\documentclass{article}

\PassOptionsToPackage{numbers,compress,sort}{natbib}

\usepackage[final]{neurips_2021}

\usepackage{commands}

\usepackage[utf8]{inputenc} 
\usepackage[T1]{fontenc}    
\usepackage{hyperref}       
\usepackage{url}            
\usepackage{booktabs}       
\usepackage{amsfonts}       
\usepackage{nicefrac}       
\usepackage{microtype}      
\usepackage{xcolor}         
\usepackage{graphicx}       
\usepackage[export]{adjustbox}
\usepackage{float}
\usepackage{verbatim}
\usepackage{multirow}
\usepackage{mathtools}
\usepackage{siunitx}
\graphicspath{{figures/}}

\usepackage{enumitem}
\setlist{leftmargin=0.5cm,itemsep=2pt,parsep=0pt}

\title{Support vector machines and linear regression coincide with very high-dimensional features}

\author{%
  Navid Ardeshir* \\
  Dept.\ of Statistics\\
  Columbia University\\
  \texttt{na2844@columbia.edu} \\
  \And
   Clayton Sanford* \\
   Dept.\ of Computer Science \\
   Columbia University\\
   \texttt{clayton@cs.columbia.edu} \\
   \And
   Daniel Hsu \\
   Dept.\ of Computer Science \\
   Columbia University\\
   \texttt{djhsu@cs.columbia.edu} \\
}

\begin{document}

\maketitle

\begin{abstract}
The support vector machine (SVM) and minimum Euclidean norm least squares regression are two fundamentally different approaches to fitting linear models, but they have recently been connected in models for very high-dimensional data through a phenomenon of support vector proliferation, where every training example used to fit an SVM becomes a support vector.
In this paper, we explore the generality of this phenomenon and make the following contributions.
First, we prove a super-linear lower bound on the dimension (in terms of sample size) required for support vector proliferation in independent feature models, matching the upper bounds from previous works.
We further identify a sharp phase transition in Gaussian feature models, bound the width of this transition, and give experimental support for its universality.
Finally, we hypothesize that this phase transition occurs only in much higher-dimensional settings in the $\ell_1$ variant of the SVM, and we present a new geometric characterization of the problem that may elucidate this phenomenon for the general $\ell_p$ case. 
\end{abstract}


\section{Introduction}

The \emph{support vector machine} (SVM) and \emph{ordinary least squares} (OLS) are well-weathered approaches to fitting linear models, but they are associated with different learning tasks: classification and regression.
In this paper, we study the case in which the models return exactly the same hypothesis for sufficiently high-dimensional data.

The hard-margin SVM is a linear classification model that finds the separating hyperplane that maximizes the minimum margin of error for every training sample. 
If the training data $(\bx_1, y_1), \dots, (\bx_n, y_n) \in \R^d \times \flip$ are linearly separable, then the resulting linear classifier is $x \mapsto \sign{x^\T w_{\svm}}$, where $w_{\svm}$ is the solution to the following optimization problem: 
\begin{equation}\label{eq:svm-primal-l2}
    w_{\svm} = \argmin_{w \in \R^d} \norm[2]{w} \quad \text{such that} \quad y_i w^\T \bx_i \geq 1, \ \forall i \in [n].
\end{equation}
An example $\bx_i$ is a \emph{support vector} if the corresponding constraint is satisfied with equality, and the optimal solution $w_{\svm}$ is a linear combination of these support vectors.

Ordinary least squares regression finds the linear function that best fits the training data $(\bx_1, y_1), \dots, (\bx_n, y_n) \in \R^d \times \R$ according to the sum of squared errors.
When the solution is not unique, it is natural to take the solution of minimum Euclidean norm; this is the convention we adopt.
Taking $\bX := [\bx_1 | \dots | \bx_n]^\T \in \R^{n \times d}$ and $y := (y_1, \dots, y_n)$, the solution is the hypothesis $x \mapsto w_{\ls}^\T x$ where $w_{\ls}$ is the solution to the following:
   $w_{\ls} = \argmin_{w \in \R^d} \norm[2]{w}$ such that
  $\bX^\T \bX w = \bX^\T y$.
  In many high-dimensional settings (e.g., where $\bX$ has full row rank), the solution may in fact \emph{interpolate} the training data, i.e., 
\begin{equation}\label{eq:ls-l2-overparam}
    w_{\ls} = \argmin_{w \in \R^d} \norm[2]{w}  \quad \text{such that} \quad w^\T \bx_i = y_i, \ \forall i \in [n].
\end{equation}

Although the optimization problems in~\eqref{eq:svm-primal-l2} and~\eqref{eq:ls-l2-overparam} are very different,
they have been observed to coincide in very high-dimensional regimes.
The study of this \emph{support vector proliferation} (SVP) phenomenon---in which every training example is a support vector---was recently initiated by \citet{mnsbhs20} and \citet{hmx20}.
Roughly speaking, they show that SVP occurs when $d = \Omega( n \log n)$ for a broad class of sample distributions, and that SVP does not occur when $d = O(n)$ in an idealized isotropic Gaussian case.

SVP is a phenomenon that connects linear classification and linear regression, topics that have received renewed attention due to the break-down of classical analyses of these methods in high-dimensions.
For instance, some analyses of SVM that are based on the number of support vectors become vacuous when this number becomes large~\cite[e.g.,][]{vapnik1995nature,graepel2005pac,germain2011pac}.
Similarly, overparameterized linear regression is typically only studied in noisy settings with explicit regularization.
It was not until recently that SVM and OLS have been meaningfully analyzed in these regimes (see Section~\ref{sec:related}), and the connection between the two approaches via SVP has played an important analytical role~\cite{mnsbhs20,wang2021binary,cao2021risk}.



In this work, we further examine support vector proliferation with the goal of broadly understanding when and why SVMs and OLS coincide. 
We pose and study the following questions:
\begin{enumerate}
    \item \emph{How general is the SVP phenomena? 
    What relationship between $d$ and $n$ determines if the solutions to \eqref{eq:svm-primal-l2} and \eqref{eq:ls-l2-overparam} coincide?}

    \smallskip
    We close the $\log n$ gap from the prior work of \citet{hmx20} by showing that $d \gtrsim n \log n$ is \emph{necessary} for SVP to occur under a model of independent subgaussian features, even with constant probability.
    Our lower-bounds hold for a broad class of distributions over $\bx_i$, and they match the upper-bounds from \cite{hmx20}.
    This demonstrates that SVP is extremely unlikely to occur in the much-studied $d = \Theta(n)$ setting.

    \item \emph{Is there a sharp threshold separating the occurrence and non-occurrence of this phenomenon?
    Is this threshold universal across all ``reasonable'' distributions over each $\bx_i$?}

    \smallskip
    We hypothesize that a sharp phase transition occurs at $d = 2 n \log n$.
    We rigorously prove this hypothesis for isotropic Gaussian features and quantitatively bound the width of the transition.
    We experimentally observe the same transition for a wide range of other distributions.

    \item \emph{Is support vector proliferation specific to the $\ell_2$ SVM problem? 
    If \eqref{eq:svm-primal-l2} and \eqref{eq:ls-l2-overparam} are generalized to instead minimize $\ell_p$ norms, does this still occur at the same rate?}

    \smallskip
    We re-frame this question with a geometric characterization of the dual of the SVM optimization problem for $\ell_p$ norms.
    We conjecture that a similar phase transition occurs for $\ell_1$, but also that it requires much larger dimension $d$; this is supported by preliminary experiments.

\end{enumerate}









\subsection{Outline of our results}

Section~\ref{sec:prelim} introduces the SVM and OLS approaches in full generality, our $\lambda$-anisotropic subgaussian data model, and prior results about SVP.
Several equivalent characterizations of SVP are established (Proposition~\ref{prop:equiv}) for use in subsequent sections.

Section~\ref{sec:nonasymp} characterizes when SVP \emph{does not} occur for a broad range of distributions (Theorem~\ref{thm:anisotropic}).
Our lower-bound on the dimension required for SVP matches the upper-bounds from \cite{hmx20} in the isotropic Gaussian setting, resolving the open question from that work, and also gives new lower-bounds for anisotropic cases.
The proof works by tightly controlling the spectrum of the Gram matrix and establishing anti-concentration via the Berry-Esseen Theorem.

Section~\ref{sec:asymptotic} establishes a sharp threshold of $d = 2 n \ln n$ for SVP in the case of isotropic Gaussian samples, and also characterize the width of the phase transition (Theorem~\ref{thm:asymptotic}).

Section~\ref{sec:experiments} provides empirical evidence that the sharp threshold observed in Section~\ref{sec:asymptotic} holds for a wide range of random variables.
Rigorous statistical methodology inspired by
\citet{donoho2009observed} is used to test our ``universality hypothesis'' that the probability of SVP does not depend on the underlying sample distribution as $d$ and $n$ become large.

Section~\ref{sec:l1} asks the questions about SVP from the preceding sections in the context of $\ell_1$-SVM and minimum $\ell_1$-norm interpolation.
Specifically, the SVP threshold for $\ell_1$ is conjectured to occur for $d = \omega(n \log n)$.
Evidence for this conjecture is provided in a simulation study and in geometric arguments about random linear programs.






\subsection{Related work}
\label{sec:related}


\paragraph{Prior works connecting SVP and generalization.}

\citet{mnsbhs20} initiate the study of SVP in part to facilitate generalization analysis of the SVM in very high-dimensional settings.
Their work, as well as the contemporaneous work of \citet{cl20}, shows that the SVM enjoys low test error in certain regimes where classical learning-theoretic analyses would otherwise yield vacuous error bounds.
(In fact, one of the settings in \cite{cl20} requires polynomially-higher dimension than is typically studied: $d = \Omega(n^2\log n)$.)
The coincidence between SVM and OLS identified by \citeauthor{mnsbhs20} was also more recently used by \citet{wang2021binary} and \citet{cao2021risk} for analyses of linear classification in very high dimensions under different data distributions.

The generalization analysis of \citeauthor{mnsbhs20} concerns a data model inspired by the spiked covariance model of \citet{wang2017asymptotics}.
They identify a regime of overparameterization where the hard-margin SVM classifier has good generalization (i.e., classification risk going to 0) even when all the training samples are supoort vectors. Our new lower bound can be regarded as establishing a limit on this approach to the analysis of SVM; specifically, if the (effective) dimension is not sufficiently large, the OLS and SVM solutions may not coincide.

\paragraph{Prior analyses of number of support vectors.}
Besides its relevance to generalization analysis, the number of support vectors in an SVM model is an interesting quantity to study in its own right.
\citet{hmx20} sharpen and extend the analysis of \citet{mnsbhs20} about SVP in the independent features model that we also adopt.
They prove that SVM on $n$ samples with $d$ independent subgaussian components coincides with OLS when $d = \Omega(n \log n)$ with probability tending to $1$.
They also give a converse result stating that the coincidence fails with constant probability when $d = O(n)$ in the isotropic Gaussian feature model.
(We give these results here as Theorems~\ref{thm:hmx-ub} and \ref{thm:hmx-lb} respectively.)
Our results generalize and tighten the latter bound to tell an asymptotically sharp story about the phase transition for both isotropic and anisotropic random vectors with subgaussian components.
Our specific analysis for the isotropic Gaussian case gives the exact point of the phase transition.

The number of support vectors is also studied in the context of variants of SVM~\cite{steinwart2003sparseness,bartlett2007sparseness}, including the soft-margin SVM~\cite{cortes1995support} and the $\nu$-SVM~\cite{scholkopf2000new}.
In these cases, the asymptotic number of support vectors is shown to be related to the noise rate in the problem. The setups we study are linearly separable, which makes it possible to study the hard-margin SVM (without regularization).
The hard-margin SVM is also of interest because it captures the implicit bias of gradient descent on the logistic loss objective for linear predictors~\cite{shmgs18,ji2019implicit}.



Phase transitions have been studied in the context of linear classification~\cite{cover1965geometrical,theisen2021good,candes2018phase,liang2020precise}, and SVMs in particular~\cite{dietrich1999statistical,malzahn2005statistical,liu2019exact,bg01}, but most study qualitative changes in behavior other than support vector proliferation.
The most relevant is the study of \citet{bg01}, who employ techniques from statistical physics to show the existence of phase transitions for the generalization error, margin size, and number of support vectors as $n$ and $d = \Theta(n)$ become arbitrarily large.
While they characterize the fraction of samples that are support vectors, they do not address our question about when \emph{all} samples are support vectors, not just a large fraction.
Indeed, our results demonstrate that their regime where $d$ grows linearly with $n$ will not exhibit support vector proliferation when $n$ and $d$ in the limit.

\paragraph{Overparameterized linear regression.}
There has been a recent flurry of analyses of overparameterized linear regression models~ \cite[e.g.,][]{bhx20,bllt20,hmrt19,muthukumar2020harmless,mei2019generalization,mitra2019understanding,wang2017asymptotics,mahdaviyeh2019risk,mnsbhs20,xu2019number,huang2020dimensionality}.
Many of these analyses are carried out in the $d = \Theta(n)$ asymptotic regime, whereas our work studies a phase transition that occurs in a much higher-dimensional regime.
The notions of effective dimensions we use are present in the analyses of \citet{bllt20} and \citet{mnsbhs20}, and the latter work identifies regimes where SVM and OLS coincide and enjoy good performance for both classification and regression.


\paragraph{High-dimensional geometry and universality.}
Our conjecture about support vector proliferation for $\ell_1$-SVMs derives inspiration from studies of high-dimensional geometric phase transitions, particularly those by \cite{ab12, almt14, donoho2009counting}.
These results consider the geometry of random polytopes.
\citet{ab12} establish phase transitions on the feasibility and boundedness of the solutions to random linear programs, \citet{almt14} extend these results to characterize when $\ell_1$-norm minimizing solutions to sparse recovery problems are exactly correct, and \citet{donoho2009counting} bound the number of faces of random polytopes. 
We also borrow heavily from \cite{donoho2009observed} when designing our experiments in Section~\ref{sec:experiments} to test the universality hypothesis.


\section{Preliminaries}\label{sec:prelim}

This section introduces notation, as well as the optimization problems and data models we consider.
We also define support vector proliferation and prove the equivalence of different formulations.

\subsection{Notation}
For $\lambda\in \R_{+}^d$, we define the $\ell_2$ and $\ell_{\infty}$ \emph{dimension proxies} as $d_2 := \fracl{\norml[1]{\lambda}^2}{\norml[2]{\lambda}^2}$ and $d_{\infty} := \fracl{\norml[1]{\lambda}}{\norml[\infty]{\lambda}}$
Let $[n] := \set{1, \dots,  n}$.
For some vector $w \in \R^n$ and matrix $A \in \R^{n \times d}$, we let $w_i$ and $A_i$ denote the $i$th element of $w$ and row of $A$ respectively; likewise, we let $A_\col{j} \in \R^n$ represent the $j$th column of $A$.
We abuse notation to let $w_{\setminus i} = (w_1, \dots, w_{i-1}, w_{i+1}, \dots, w_n) \in \R^{n-1}$, $w_{[m]} = (w_1, \dots, w_m) \in \R^m$, $w_{\setminus [m]} = (w_{m+1}, \dots, w_n) \in \R^{n-m}$, and $w_{[m] \setminus i} = (w_{[m]})_{\setminus i} \in \R^{m-1}$ for $i \in [m]$ and $m \in [n]$.
Analogous notation holds for $A_{\setminus i}$, $A_{[m]}$, $A_{\setminus [m]}$, and $A_{[m] \setminus i}$.
We frequently consider the Gram matrix $\bK := \bX \bX^\T \in \R^{n\times n}$ for feature matrix $\bX \in \R^{n \times d}$; for these matrices, we let $\bK_{\setminus i} = \bX_{\setminus i} \bX_{\setminus i}^\T \in \R^{(n-1)\times(n-1)}$ and analogously define $\bK_{[m]}$, $\bK_{\setminus [m]}$, and $\bK_{[m] \setminus i}$.
Let $\mu_{\max}(A)$ and $\mu_{\min}(A)$ represent the largest and smallest eigenvalues of the matrix $A$ respectively, and let $\norml{A}$ be the operator norm of $A$.
For some vector $y \in \R^n$, we let $\diag(y) \in \R^{n \times n}$ be a diagonal matrix with $(\diag(y))_{i,i} = y_i$.
Throughout, boldface characters refer to random variables.

\subsection{Optimization problems}

We consider the hard-margin support vector machine (SVM) optimization problem and ask under what conditions one may expect all the slackness conditions to be satisfied. 
We consider training samples $(\bx_1, y_1), \dots, (\bx_n, y_n) \in \R^d \times \flip$ in a high-dimensional regime where $d \gg n$ with design matrix $\bX := [\bx_1 | \dots | \bx_n]^\T \in \R^{n \times d}$ and Gram matrix $\bK = \bX \bX^\T \in \R^{n \times n}$.
In full generality, the separating hyperplane corresponding to the $\ell_p$-SVM problem for some $p \geq 1$ is the solution to the following optimization problem:
\begin{equation}
    \tag{SVM Primal}
    \label{eq:svm-primal-lp}
    \min_{w \in \R^d} \norm[p]{w} \quad \text{such that} \quad y_i w^\T\bx_i \geq 1, \ \forall i \in [n].
\end{equation}
Our results in Sections~\ref{sec:nonasymp}--\ref{sec:experiments} concern $p=2$, and we discuss $p=1$ in Section~\ref{sec:l1}.
It is worth mentioning that feasibility is not a concern in the settings we consider.\footnote{When $d \ge n$, we are always able to find a separating hyperplane since the features are linearly independent with high probability. In fact, a theorem of \citet{cover1965geometrical} shows that feasibility holds with high probability under mild distributional assumptions even for $d > \fracl{n}{2}$.}
An example $(\bx_i, y_i)$ is called a \emph{support vector} if it lies exactly on the margin defined by separator $w$, or equivalently if $y_i w^\T {\bx_i} = 1$.
It is well-known that $w$ can be represented as a non-negative linear combination of all $y_i \bx_i$ where $\bx_i$ is a support vector~\cite{vapnik1982estimation}.

We contrast the weights of the classifier returned by \ref{eq:svm-primal-lp} with the weights of minimum $\ell_p$-norm that satisfy the normal equations of ordinary least squares (OLS).
In the case where the training data can be linearly interpolated, this optimization problem is:
\begin{equation}
    \tag{Interpolation Primal}
    \label{eq:svm-restricted-primal-lp}
    \min_{w \in \R^d} \norm[p]{w} \quad \text{such that} \quad y_i w^\T{\bx_i} = 1, \ \forall i \in [n].
\end{equation}
Per the convention mentioned in the introduction,
the solution of \eqref{eq:svm-restricted-primal-lp} when $p=2$ is referred to as ordinary least squares.
Feasibility is ensured as long as the feature vectors $\bx_i$ are linearly independent.

\subsection{Equivalent formulations of SVP}

We study the phenomenon of \emph{support vector proliferation} (SVP), i.e., the occurrence in which every example $\bx_i$ is a support vector.
Because $\bx_i$ is a support vector if $y_i w^\T{\bx_i} = 1$, this occurs if and only if the solution of \eqref{eq:svm-primal-lp} coincides exactly with that of \eqref{eq:svm-restricted-primal-lp}.
Here, we analyze those formulations to show equivalent conditions needed for SVP, which we give in Proposition~\ref{prop:equiv}.
Before presenting the proposition, we introduce the notation needed to use the alternate formulations.

We translate the relationship between the two primal optimization problems into the dual space.
Taking $\bA =  \diag(y) \bX  \in \mathbb{R}^{n \times d}$, the dual of the optimization problem \eqref{eq:svm-restricted-primal-lp} is:
\begin{equation}
    \tag{Interpolation Dual}
    \label{eq:svm-restricted-dual-lp}
    \max_{\alpha \in \mathbb{R}^{n}} \sum_{i=1}^{n}{\alpha_{i}} \quad
    \text{such that}  \quad \|\bA^\T \alpha \|_{q} \le 1 .
\end{equation}

The dual of \eqref{eq:svm-primal-lp} is \eqref{eq:svm-restricted-dual-lp} with an additional constraint that $\alpha \in \mathbb{R}_{+}^{n}$.

Let $\bT = \set{\sum_{i=1}^n a_i \bA_i : \sum_{i=1}^n a_i = 1}$ denote the affine plane passing through the rows of $\bA$, and let $\bT^+ = \set{\sum_{i=1}^n a_i \bA_i : \sum_{i=1}^n a_i = 1, \ a_i \geq 0}$ be the convex hull of the rows of $\bA$.
In addition, for $i \in [n]$, let $\bT_{\setminus i} = \setl{\sum_{i'\neq i} a_{i'} \bA_{i'} : \sum_{i'\neq i} a_{i'} = 1}$.
We denote $\proj$ as the $\ell_q$-norm projection of the origin onto $\bT$, which is uniquely defined for $1 < q < \infty$.

\begin{restatable}{proposition}{propequiv}\label{prop:equiv}
Let $1 < p < \infty$ and $q = (1-1/p)^{-1}$, and consider any $(\bX,y) \in \R^{n \times d} \times \flip^n$.
Suppose $\bK$ is invertible.
Then, the following are equivalent:
\begin{enumerate}[label=(\arabic*)]
    \item SVP occurs for $\ell_p$-SVM.
    \item The solutions $w$ to \eqref{eq:svm-primal-lp} and \eqref{eq:svm-restricted-primal-lp} are identical.
    \item The optimal solution to \eqref{eq:svm-restricted-dual-lp} lies within the interior of $\R_{+}^{d}$.
    \item $\proj \in \mathbf{T}^{+}$.
\end{enumerate}
Moreover, if $p=2$, then properties (1)--(4) are also equivalent to the following:
\begin{enumerate}[label=(\arabic*)]
    \setcounter{enumi}{4}
    \item  For all $i \in [n]$, $ y_{i} y_{\setminus i}^\T \bK_{\setminus i}^{-1} \bX_{\setminus i} \bx_{i} =  y_{i}\projj{i}^\T \bx_{i} / \norml[2]{\projj{i}}^{2} < 1.$
\end{enumerate}
\end{restatable}

This dual framework in (3) and (4) gives an alternative geometric structure to consider for this problem.
For the $\ell_{2}$-case, this formulation draws from the fact that the separating hyperplane obtained from an SVM is represented as a linear combination of support vectors.
Although the $\ell_1$ case is not technically covered by Proposition~\ref{prop:equiv} (due to the non-strict convexity of $\ell_1$ norm), our analysis still gives useful insights, and we explore this case specifically in Section~\ref{sec:l1}. 

We prove Proposition~\ref{prop:equiv} in Appendix~\ref{app:prelim-proof}.
The equivalence between (1) and (5) in the $p=2$ case was proved by \citet[Lemma 1]{hmx20}.
Our alternative proof is based on establishing the equivalence of (4) and (5) and draws heavily from our geometric formulation of SVP.

\subsection{Data model}
We use the data model of \citet{hmx20}, where every labeled sample $(\bx_i, y_i)$ has $\bx_i$ drawn from an anisotropic subgaussian distribution with independent components and arbitrary fixed labels $y_i$.

\begin{definition}
    For some $\lambda\in \R_{\geq 0}^d$, we say $(\bX, y) \in \R^{n \times d} \times \flip^n$ (as well as $(\bX,\bZ,y)\in \R^{n \times d} \times \R^{n \times d} \times \flip^n$) is a \emph{$\lambda$-anisotropic subgaussian sample} if:
    $y = (y_1,\dots,y_n) \in \flip^n$ are fixed (non-random) labels;
    $\bZ := [\bz_1 |  \dots | \bz_n]^\T \in \R^{n \times d}$ is a matrix of independent $1$-subgaussian random variables with
    $\EEl{\bz_{i,j}} = 0$ and $\EEl{\bz_{i,j}^2} = 1$; and
    $\bX := [\bx_1 |  \dots | \bx_n]^\T = \bZ \  \diag(\lambda)^{1/2} \in \R^{n \times d}$.
    We say $(\bX,y)$ is an \emph{isotropic subgaussian sample} if it is $\lambda$-anisotropic for $\lambda = (1,\dotsc,1)$.
    Finally, we say $(\bX,y)$ is an \emph{isotropic Gaussian sample} if it is isotropic subgaussian and each $\bx_i \sim \mathcal{N}(0,I_d)$.
\end{definition}


%
%
We only consider fixed labels $y$ that do not depend on $\bX$. 
However, we do not consider this to be a major limitation of this work.
As discussed before, \citet{cover1965geometrical} shows that linear separability is overwhelmingly likely in the high-dimensional regimes we consider.
Moreover, our results can be extended to a setting where $\by_i = \sign{v^\T{\bx_i}}$ for some fixed weight vector $v$.

\subsection{Previous results}

We tighten and generalize the characterization of the SVP threshold by \citet{hmx20}. 
We give versions of their results that are directly comparable to our results in Sections~\ref{sec:nonasymp} and \ref{sec:asymptotic}.

\begin{theorem}
[Theorem~1 of \cite{hmx20}]
\label{thm:hmx-ub}
    Consider a $\lambda$-anisotropic subgaussian sample $(\bX, y)$ and any $\delta \in (0,1)$.
    If $d_{\infty} = \Omega(n \log (n) \log (\frac{1}{\delta}))$, then SVP occurs for $\ell_2$-SVM with probability at least $1 - \delta$.
\end{theorem}

\begin{theorem}
[Theorem~3 of \cite{hmx20}]
\label{thm:hmx-lb}
    Consider an isotropic Gaussian sample $(\bX, y)$.
    For some constant $\delta \in (0,1)$, if $d = O(n)$, then SVP occurs for $\ell_2$-SVM with probability at most $\delta$. 
\end{theorem}


They note the logarithmic separation between Theorems~\ref{thm:hmx-ub} and \ref{thm:hmx-lb} and the limitations of the data model used in Theorem~\ref{thm:hmx-lb}. 
The authors pose an improvement in generality and asymptotic tightness to their lower-bound as an open problem, which we resolve in the subsequent sections.


\section{SVP threshold for anisotropic subgaussian samples}\label{sec:nonasymp}



We closely characterize when support vector proliferation does and does not occur through the following theorem, which serves as a converse to Theorem~\ref{thm:hmx-ub}.

\begin{restatable}[Lower-bound on SVP threshold for anisotropic subgaussians]{theorem}{thmanisotropic}\label{thm:anisotropic}
    Consider a $\lambda$-anisotropic subgaussian sample $(\bX,y)$ and any $\delta \in (0, \half)$.
	For absolute constants $\Ca, \Cb, \Cc, \Cd$, assume that $\lambda$ and $n$ satisfy
\begin{equation}\label{eq:3}
	  n \geq \Ca \paren{\log \frac{1}{\delta}}^2, \quad d_2 \leq \Cb n \log n, \quad
	  d_{\infty} \geq \Cc n \log\frac{1}{\delta}, \quad \text{and} \quad
	  d_{\infty}^2 \geq  \Cd d_2 n.
	 \end{equation}
	Then, SVP occurs for $\ell_2$-SVM with probability at most $\delta$.
\end{restatable}




\begin{remark}
    If each $\bx_i$ is drawn from a Gaussian distribution, then we could instead permit $\bx_i$ to have any positive semi-definite covariance matrix $\Sigma \in \R^{d \times d}$ with eigenvalues $\lambda_1, \dots, \lambda_d$ due to rotational invariance.
\end{remark}
\begin{remark}
    In addition, the result can be generalized to subgaussian $\bx_i$ with general variance proxies $\gamma\geq1$. We present the current version for the sake of simplicity and note that the generalization is straightforward. 
\end{remark}

In the case where $(\bX,y)$ is an isotropic subgaussian sample, Theorem~\ref{thm:anisotropic} and Theorem~\ref{thm:hmx-ub} (from \cite{hmx20}) together establish that the threshold for SVP occurs at $d = \Theta(n \log n)$. 
Theorem~\ref{thm:anisotropic} sharpens and generalizes the partial converse of \cite{hmx20} given in Theorem~\ref{thm:hmx-lb}.

Theorem~\ref{thm:anisotropic} does not depend explicitly on the ambient dimension $d$; instead, it only involves the effective dimension proxies $d_2$ and $d_\infty$, which can be finite even if $d$ is infinite.
Thus, the result readily extends to infinite-dimensional Hilbert spaces.

We prove the theorem in Appendix~\ref{app:lemmas} and briefly summarize the techniques here.
By Proposition~\ref{prop:equiv}, it suffices to show that with probability $1 - \delta$, $\bK$ is invertible and
\begin{equation}\label{eq:1}
\max_{i \in [n]} y_i y_{\setminus i}^\T \bK_{\setminus i}^{-1} \bX_{\setminus i} \bx_i \geq 1,
\end{equation}
where $\bK_{\setminus i} := \bX_{\setminus i} \bX_{\setminus i}^\T$.
This same equivalence underlies the proof of Theorem~\ref{thm:hmx-lb} in \cite{hmx20}.
However, their application of this equivalence is limited because they avoid issues of dependence between random variables by instead lower-bounding the probability that $y_1 y_{\setminus 1}^\T \bK_{\setminus 1}^{-1} \bX_{\setminus 1} \bx_1 \geq 1$.
This forces their bound to hold only when $d = O(n)$.
We obtain a tighter bound by separating the first $m$ samples (denoted $\bX_{[m]}$) for some carefully chosen $m$ and relating the term to the maximum of $m$ independent random variables.
To do so, we lower-bound the left-hand side of \eqref{eq:1} with the following decomposition:
\[\max_{i \in [m]} \bracket{y_i y_{\setminus i}^\T \paren{\bK_{\setminus i}^{-1}-\frac{1}{\norm[1]{\lambda}} I_{n-1}} \bX_{\setminus i} \bx_i + \frac{1}{\norm[1]{\lambda}}y_i y_{[m] \setminus i}^\T \bX_{[m] \setminus i} \bx_i + \frac{1}{\norm[1]{\lambda}}y_i y_{\setminus [m]}^\T \bX_{\setminus [m]} \bx_i}.\]

We prove that this decomposition (and hence, also \eqref{eq:1}) is at least 1 with probability $1-\delta$ by lower-bounding the three terms with Lemmas~\ref{lemma:e1}, \ref{lemma:e2}, and \ref{lemma:e3} (given in Appendix~\ref{app:lemmas}).
We bound the first two terms for all $i\in [m]$ by employing standard concentration bounds for subgaussian and subexponential random variables and by tightly controlling the spectrum of $\bK_{\setminus i}$.
To bound the third term, we relate the quantity for each $i \in [m]$ to an independent univariate Gaussian with the Berry-Esseen theorem and apply standard lower-bounds on the maximum of $m$ independent Gaussians.


The assumptions in \eqref{eq:3} are all intuitive and necessary for our arguments.
The first assumption ensures that enough samples are drawn for high-probability concentration bounds to exist over collections of $n$ variables.
The second assumption guarantees the sub-sample size $m$ is sufficiently large to have predictable statistical properties; this is asymptotically tight with its counterpart in Theorem~\ref{thm:hmx-ub} up to a factor of $\log \frac{1}{\delta}$.
The third ensures that the variance of each $y_i y_{\setminus i}^\T \bK_{\setminus i}^{-1} \bX_{\setminus i} \bx_i$ term is sufficiently small.
The fourth assumption rules out $\lambda$-anisotropic subgaussian distributions with $\norml[2]{\lambda}^2 \ll \norml[\infty]{\lambda}^2 n$, where a single component of each $\bx_i$ is disproportionately large relative to others and causes unfavorable anti-concentration properties.


\section{Exact asymptotic threshold for Gaussian samples}\label{sec:asymptotic}

Section~\ref{sec:nonasymp} shows the existence of a change in model behavior when $d = \Theta(n \log n)$ without identifying a precise threshold where this phase transition appears.
Here, we refine that analysis for the isotropic Gaussian case to find such an exact threshold. 
That is, if $d = 2\tau n \log n$, as $n$ becomes large, SVP will occur when $\tau > 1$ and will not occur when $\tau < 1$. Roughly speaking, this phenomenon stems from the fact that terms in \eqref{eq:1} are weakly correlated, which causes \eqref{eq:1} to behave similarly to a maximum of independent Gaussians.
Furthermore, we characterize the rate at which the phase transition sharpens.
The following theorem shows that if the convergence $\tau \to 1$ is slow enough, then the asymptotic probabilities of SVP are degenerate and the width of the transition is bounded.


\begin{restatable}[Sharp SVP phase transition]{theorem}{thmasymptotic}\label{thm:asymptotic}
  Let $(\bX,y)$ be an isotropic Gaussian sample.
  Let $(\eps_{n})_{n \ge 1}$ be any sequence of positive real numbers such that $\limsup_{n\to\infty} \eps_n < 2-c_1$ for some $c_1>0$ and $\liminf_{n \to \infty}{\eps_{n}\sqrt{\log{n}}} > C_2$ for some $C_2>0$ depending only on $c_1$.
  Then,
  \begin{equation*}
    \lim_{n \to \infty} \prl{\text{SVP occurs for $\ell_{2}$-SVM}} =
    \begin{cases}
      0 & \text{if $d = (2 - \eps_{n}) n \ln n $} , \\
      1 & \text{if $d = (2 + \eps_{n}) n \ln n $} .
    \end{cases}
  \end{equation*}
\end{restatable}

\begin{remark}\label{rmk:transition-width}
Theorem~\ref{thm:asymptotic} characterizes the width of the phase transition: the difference $w_n$ between the values of $d$ where the probability of SVP is (say) $0.9$ and $0.1$ satisfies $w_n = O(n \sqrt{\log{n}})$.
\end{remark}


It remains an open problem to determine if this transition width estimate is sharp.
Specifically, the bound can be sharpened by exhibiting some sequence $\eps_n$ for which the asymptotic probability of support-vector proliferation is non-degenerate.

The proof of Theorem~\ref{thm:asymptotic} is given in Appendix~\ref{app:asymp-proofs}.
In the case where $d = (2  - \eps_n) n \ln n$, the proof mirrors that of Theorem~\ref{thm:anisotropic}, but deviates in the final step by using the limiting distribution of the maximum of independent Gaussians.
When $d = (2 + \eps_n) n \ln n$, we follow the basic argument in the proof of Theorem~\ref{thm:hmx-ub} from \cite{hmx20}, but we sharpen the analysis by taking advantage of Gaussianity to find the limiting probability as $n\to\infty$.


\section{Experimental validation of SVP phase transition and universality}\label{sec:experiments}

While Theorem~\ref{thm:asymptotic} identifies the exact SVP  phase transition for only isotropic Gaussian samples, we demonstrate experimentally that a similarly sharp cutoff occurs for a broader category of data distributions.
These experiments suggest that the phase transition phenomenon extends beyond the distributions with independent subgaussian components considered in Theorem~\ref{thm:anisotropic}, and that it occurs at the same location ($d = 2 n \log n$), with the transition sharpening as $n \to \infty$.

%
%

Our simulation procedure is as follows.
We generate data sets $(\bX,y) \in \R^{n \times d} \times \flip^n$, where $y \in \flip^n$ is a fixed vector of labels with exactly $n/2$ positive labels, and $\bx_1,\dotsc,\bx_n \sim_{\text{i.i.d.}} \distr^{\otimes d}$ where $\distr$ is one of six sample distributions on $\R$.
For each data set, we check for SVP by solving the problem in~\eqref{eq:svm-restricted-primal-lp}, and checking if its solution additionally satisfies the constraints from~\eqref{eq:svm-primal-lp}.
Let $\hat\bp \coloneqq \hat\bp(n,d; \distr,M)$ denote the observed frequency of SVP in $M=400$ independent trials when $\bX \in \R^{n \times d}$ is generated using $\distr$.
(Full details are given in Appendix~\ref{app:exp-procedure}.)

Figure~\ref{fig:heatmap-different-distributions} shows a heat map of $\hat\bp(n,d; \distr,M)$ with $M=400$.
The striking similarity across the distributions suggests that SVP is a universal phenomenon for a broad class of sample distributions that vary qualitatively in different aspects: biased vs.\ unbiased, continuous vs.\ discrete, bounded vs.\ unbounded, and subgaussian vs.\ non-subgaussian.
Moreover, the boundary at which the sharp transition occurs is visibly indistinguishable across the different sample distributions.

\begin{figure}[t]
    \centering
    \includegraphics[scale = 0.8]{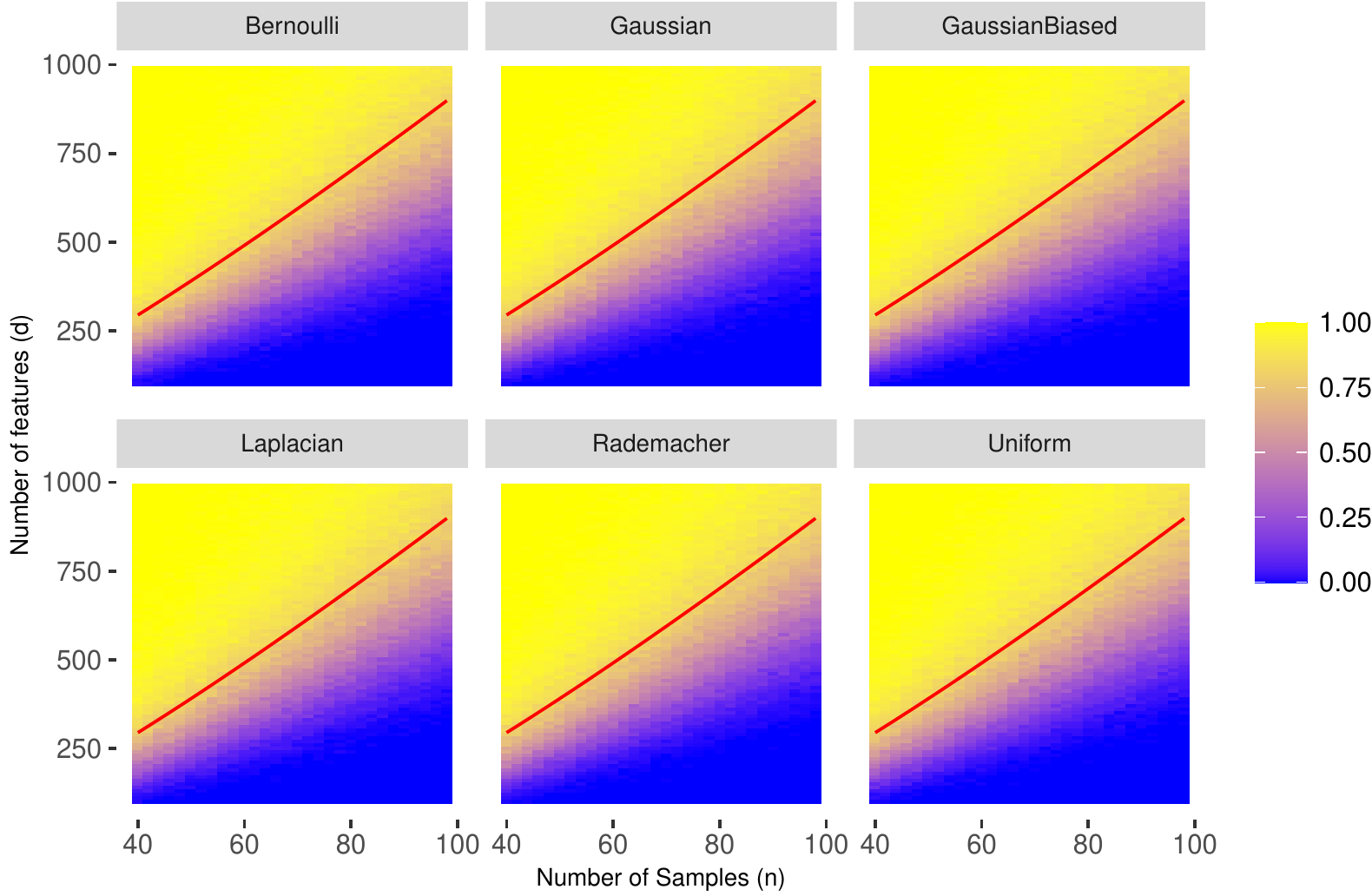}
    \caption{%
    \label{fig:heatmap-different-distributions}%
    The fraction of $M = 400$ trials where support vector proliferation occurs for $n$ samples, $d$ features, and six different sample distributions $\distr$.
    All distributions demonstrate sharp phase transitions near the theoretical boundary $n \mapsto 2n \log n$, illustrated by the red curve.} 
\end{figure}


We also investigate the universality of SVP using statistical methodology inspired by \citet{donoho2009observed}.
Specifically, we model $\hat\bp$ using Probit regression to test our \emph{universality hypothesis}: that the occurrence of SVP for $\ell_2$-SVM on data with independent features matches the behavior under Gaussian features as $n$ and $d$ grow large.
Our model is
$M \cdot \hat\bp \sim \Binomial(p(n,d;\distr), M)$,
where
\begin{align*}
  p(n,d ; \distr)
  & = \Phi\paren{ \mu^{(0)}(n,\distr) + \mu^{(1)}(n, \distr) \times \tau + \mu^{(2)}(n,\distr) \times \log \tau  } \\
  \text{with} \quad
  \mu^{(i)}(n, \distr)
  & = \mu^{(i)}_{0}(\distr) + \frac{\mu^{(i)}_{1}(\distr)}{ \sqrt{n} } .
\end{align*}
Here,
$\Phi(t)$ is the standard normal distribution function,
$\tau = d / (2n \log n)$,
and the model parameters are $\mu^{(i)}_0(\distr)$ and $\mu^{(i)}_1(\distr)$ for $i \in \{0,1,2\}$ and the six different distributions $\distr$ (shown in Table~\ref{tab:dist} in Appendix~\ref{app:exp-procedure}).
Figure~\ref{fig:constant-slice} visualizes the fitted Probit function $p$ for fixed $n$ and $\tau$ and demonstrates that the model provides a very accurate approximation of $\hat{\mathbf{p}}$. 

The universality hypothesis corresponds to the model in which the parameters $\mu^{(i)}_0(\distr)$ are ``tied together'' (i.e., forced to be the same) for all distributions $\distr$.
That is, only the parameters scaled down by a factor of $\sqrt{n}$, $\mu^{(i)}_1(\distr)$, are allowed to vary with $\distr$.
The scaling ensures that their effect tends to zero as $n \to \infty$.
The alternative (non-universality) hypothesis corresponds to the model in which all parameters
(both $\mu^{(i)}_0(n,\distr)$ and $\mu^{(i)}_1(n,\distr)$ for each $i$) are allowed to vary with $\distr$.
We compare the models' goodness-of-fit using analysis of deviance~\cite{hastie2017generalized}.
Our main finding is that the experimental data are consistent with the universality hypothesis (and also that we can reject a null hypothesis in which all parameters are ``tied together'' for all $\distr$).
The details and model diagnostics are given in Appendix~\ref{app-sub:universality}.

\begin{figure}[t]
    \centering
    \begin{tabular}{cc}
    \includegraphics[width=0.48\textwidth]{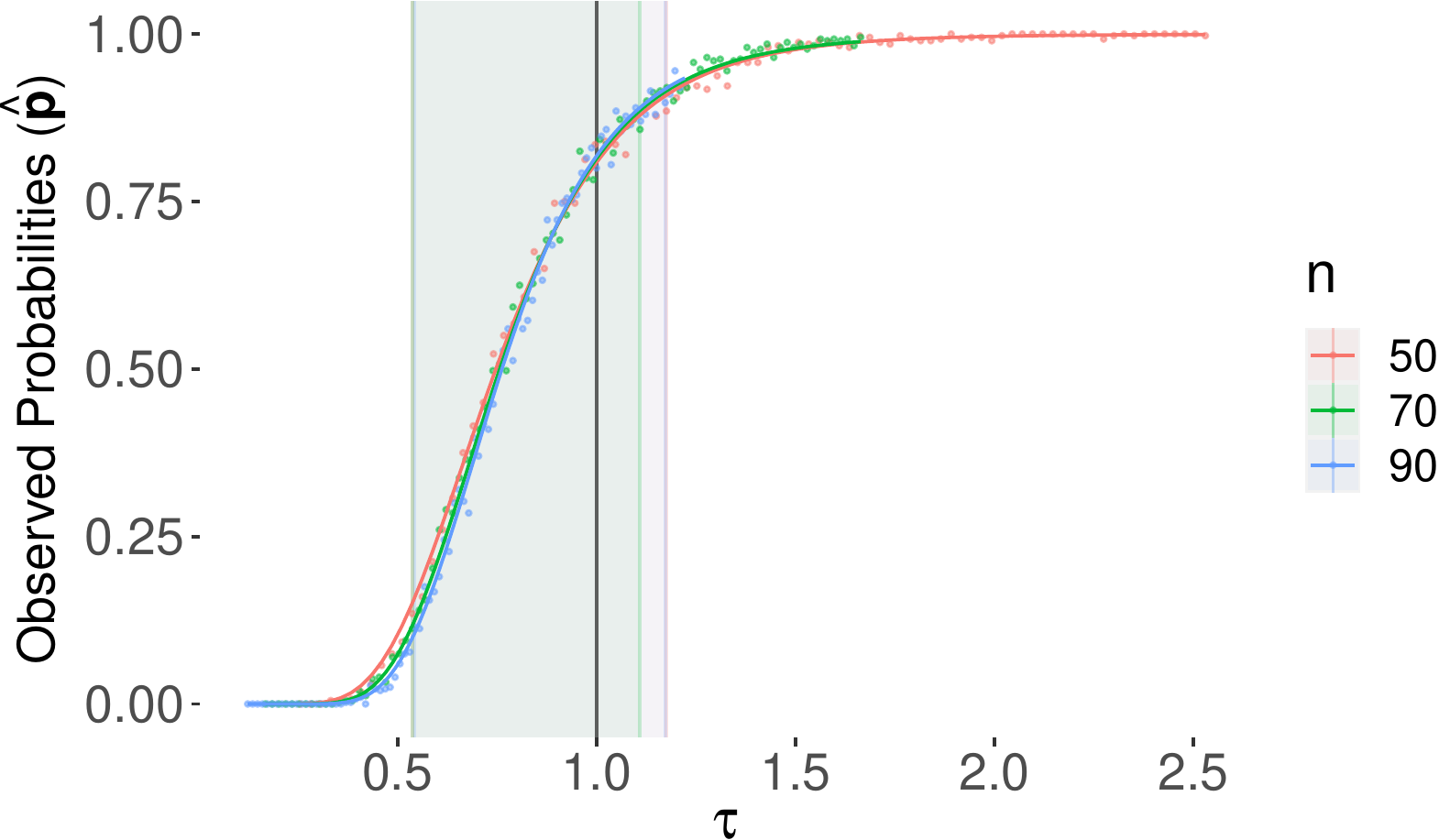} &
    \includegraphics[width=0.48\textwidth]{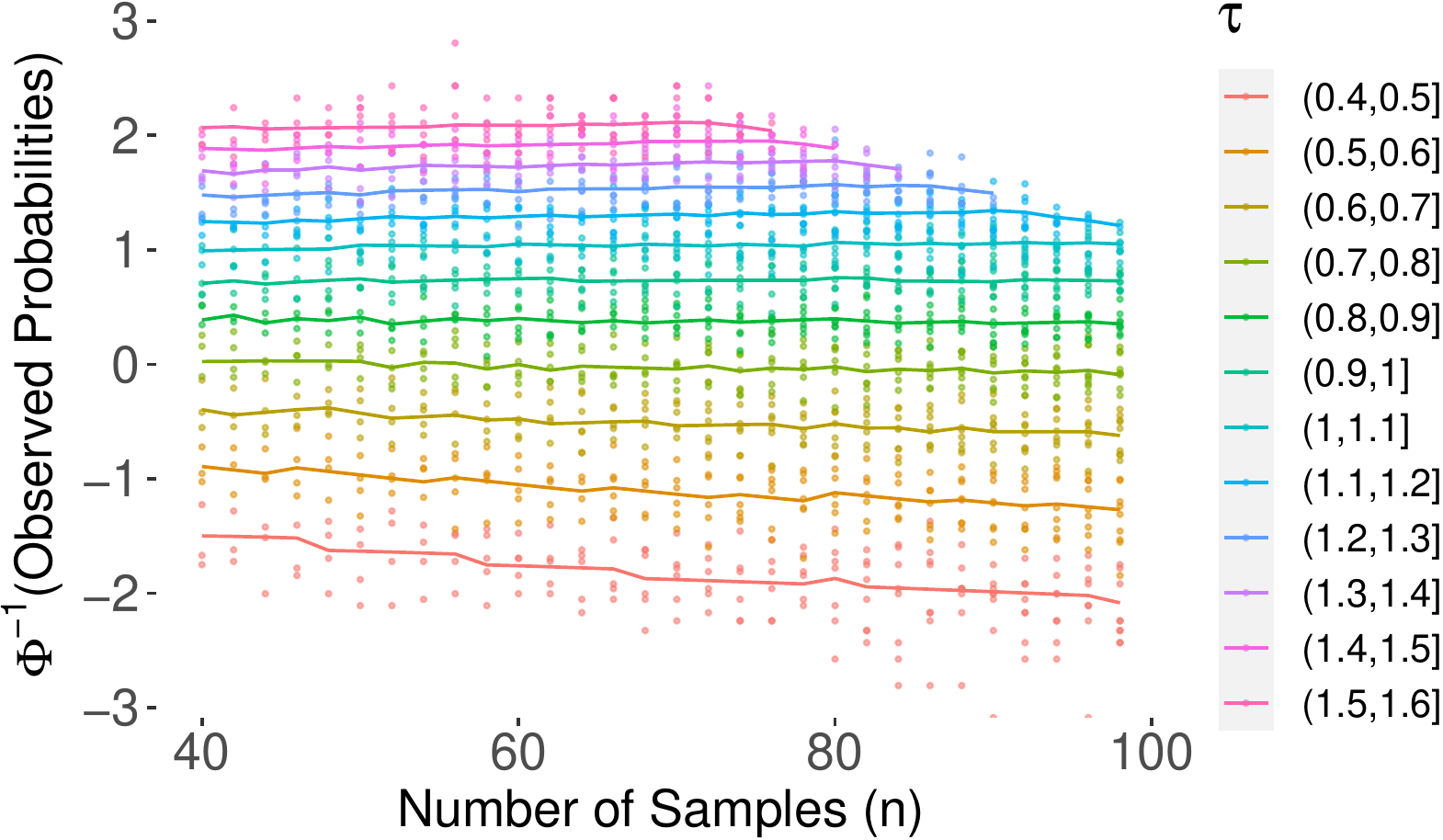}
    \\
    (a) & (b)
    \end{tabular}
    \caption{%
    \label{fig:constant-slice}%
    Visualizations of SVP frequencies for constant slices of $n$ and $\tau$ for $d = 2\tau n \log n$.
    \emph{Left panel (a):}
    The points are $(\tau,\hat\bp)$ from the Gaussian samples, for fixing $n \in \{50,70,90\}$.
    The black vertical line corresponds to $\tau = 1$.
    The Probit model's predictions are overlaid, and shaded regions correspond to $\tau$ for which the model's predicted probabilities are between $0.1$ and $0.9$.
    \emph{Right panel (b):}
    The points show $(n,\Phi^{-1}(\hat\bp))$ from a Gaussian distribution, fixing $\tau$ to lie in one of 12 different intervals.
    The Probit model's predictions (averaged over all $\tau$ within an interval) are overlaid.
}
\end{figure}



Finally, in Appendix~\ref{app-sub:width},
we provide empirical support for the generality of Remark~\ref{rmk:transition-width}, namely that the transition width is roughly $n \sqrt{\log n}$ for data models other than Gaussian ensembles.

\section{SVP phase transition for $\ell_1$-SVMs?}\label{sec:l1}
Because both the SVM and linear regression problems can be formulated for general $\ell_p$-norms, we can ask similar questions about when their solutions coincide.
Here, we examine the $\ell_1$ case: the coincidence of SVM with an $\ell_1$-penalty and $\ell_1$-norm minimizing interpolation (also called Basis Pursuit~\cite{chen1994basis}).
Linear models with $\ell_1$ regularization are often motivated by the desire for sparse weight vectors~\cite[e.g.,][]{chen1994basis,tibshirani1996regression,ng2004feature,muthukumar2020harmless}.

Based on experimental evidence and the differences in high-dimensional geometry between $\ell_{\infty}$ and $\ell_2$ balls, we conjecture that SVP for $\ell_1$-SVMs only occurs in a much higher-dimensional regime.
\begin{conjecture}\label{main-conjecture}
  Let $(\bX,y) \in \R^{n \times d} \times \R^n$ be an isotropic Gaussian sample. 
Then, the probability of SVP occurring for an $\ell_{1}$-SVM with $\bX$ and $y$ undergoes a phase transition around $d = f(n)$, for some $f(n) = \omega(n \log n)$. Formally, there exist positive constants $c$ and $c'$ with $c \leq c'$ such that
\[
\lim_{n \to \infty} \pr{ \text{SVP occurs for $\ell_{1}$-SVM} } = \begin{cases}
  0  &  \text{if $d < c f(n)$} , \\
  1 &  \text{if $d > c' f(n)$} .
\end{cases} \]
\end{conjecture}

\begin{figure}[t]
    \centering
    \includegraphics[width=\textwidth]{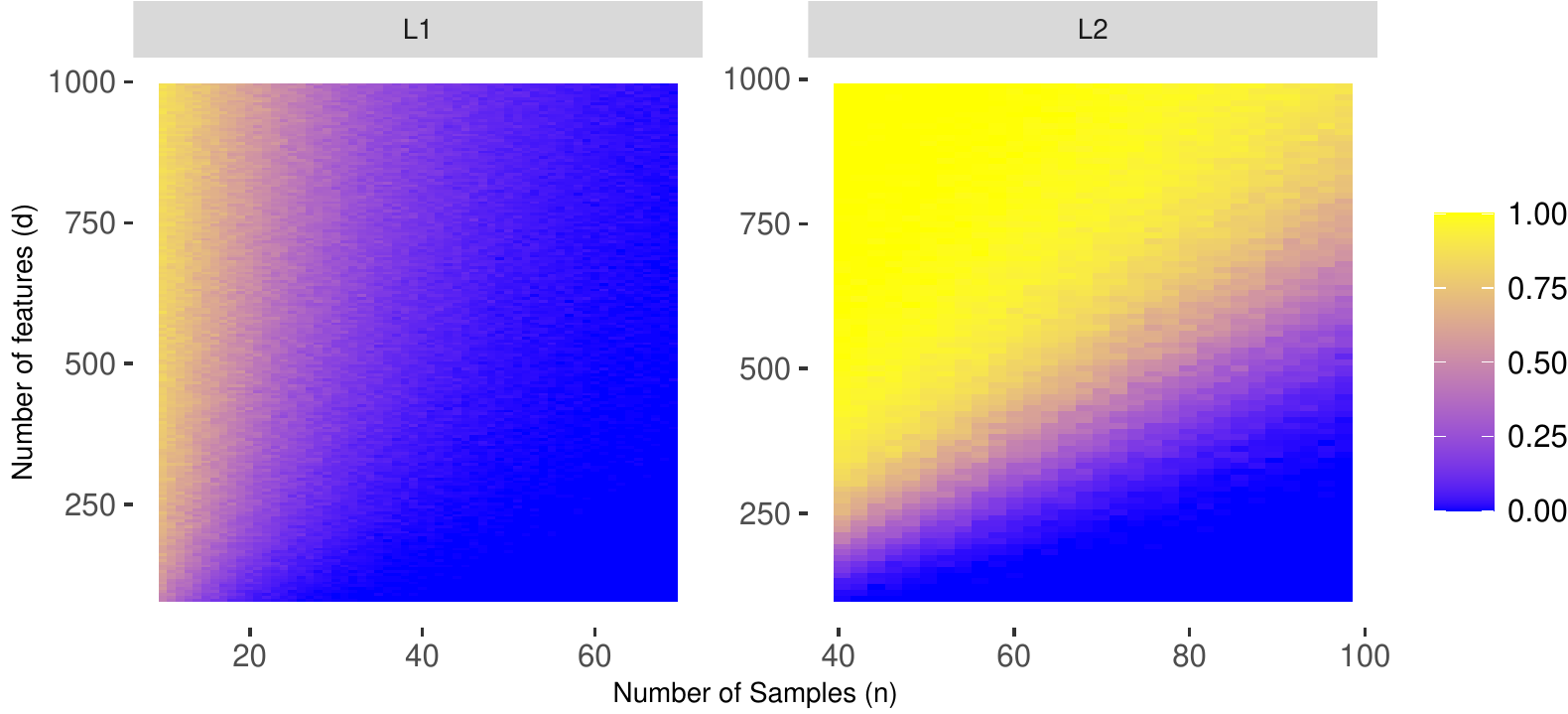}
    \caption{
    The observed probabilities of support vector proliferation for $\ell_1$- and $\ell_2$-SVMs for $d$-dimensional isotropic Gaussian samples of size $n$. 
    }
    \label{fig:l1-l2-compare}
\end{figure}

Conjecture~\ref{main-conjecture} is consistent with our preliminary experimental findings, summarized in Figure~\ref{fig:l1-l2-compare}.
It shows larger values of $d$ relative to $n$ are needed to ensure SVP for $\ell_1$-SVMs and that the transition appears to be less sharp.
Indeed, the experiments indicate that the true phase transition may occur when $d$ is asymptotically \emph{much} larger than $n \log n$.
They do not rule out the possibility that the transition may even require $d = \exp(\Omega(n))$.
Further experimental details are given in Appendix~\ref{app-sub:l1-exp}.


Answering whether support vector proliferation occurs in the $\ell_1$ case is equivalent to determining whether the optimal solution $\alpha^*$ to \eqref{eq:svm-restricted-dual-lp} problem lies in the positive orthant $\R_+^n$.\footnote{While Proposition~\ref{prop:equiv} does not imply this equivalence for $\ell_1$-SVMs for arbitrary data, the results of the proposition are valid for isotropic Gaussian samples, because the corresponding $\ell_\infty$ projection $\proj$ in those cases is well-defined almost surely.}
In the $\ell_{1}$ case, we have $q = \infty$, so solving the problem amounts to characterizing the solutions to the following linear program for data matrix $\bA =\diag(y) \bX \in \R^{n \times d}$:
\begin{equation}
    \tag{Dual L1}
    \label{eq:svm-dual-l1}
    \begin{aligned}
    \max_{\alpha \in \R^n} & \quad \sum_{i=1}^{n}{\alpha_{i}}
    &
    \text{s.t.} & \quad -\mathbf{1} \leq  \bA^\T \alpha \leq \mathbf{1}.
    \end{aligned}
\end{equation}
There is a line of work that gives high probability guarantees about whether related random linear programs are feasible and where their solutions reside~\cite{ab12, almt14}.
Similar analyses of random linear programs may be useful for understanding how large $d$ must be to have $\alpha^* \in \R^n_+$, and we carry out a preliminary characterization in Appendix~\ref{app:l1}. 

Conjecture~\ref{main-conjecture} and the other questions raised in this work point to a broader scope of investigations about high-dimensional phenomena and universality concerning optimization problems commonly used in machine learning and statistics.
Our results, along with those from prior works, provide new analytic and empirical approaches that may prove useful in tackling these questions.



\begin{ack}
D.~Hsu acknowledges support from
NSF grants CCF-1740833 and IIS-1563785,
NASA ATP grant 80NSSC18K109,
and a Sloan Research Fellowship.
C.~Sanford acknowledges support from
NSF grant CCF-1563155
and a Google Faculty Research Award to D.~Hsu.
N.~Ardeshir acknowledges support from Columbia Statistics Department.
This material is based upon work supported by the National Science Foundation under grant numbers listed above.
Any opinions, findings and conclusions or recommendations expressed in this material are those of the authors and do not necessarily reflect the views of the National Science Foundation.
We acknowledge computing resources from Columbia University's Shared Research Computing Facility project, which is supported by NIH Research Facility Improvement Grant 1G20RR030893-01, and associated funds from the New York State Empire State Development, Division of Science Technology and Innovation (NYSTAR) Contract C090171, both awarded April 15, 2010.
\end{ack}

\bibliographystyle{plainnat}

\bibliography{bib}

\begin{thebibliography}{48}
\providecommand{\natexlab}[1]{#1}
\providecommand{\url}[1]{\texttt{#1}}
\expandafter\ifx\csname urlstyle\endcsname\relax
  \providecommand{\doi}[1]{doi: #1}\else
  \providecommand{\doi}{doi: \begingroup \urlstyle{rm}\Url}\fi

\bibitem[Amelunxen and B{\"u}rgisser(2012)]{ab12}
Dennis Amelunxen and Peter B{\"u}rgisser.
\newblock Intrinsic volumes of symmetric cones.
\newblock \emph{arXiv preprint arXiv:1205.1863}, 2012.

\bibitem[Amelunxen et~al.(2014)Amelunxen, Lotz, McCoy, and Tropp]{almt14}
Dennis Amelunxen, Martin Lotz, Michael~B McCoy, and Joel~A Tropp.
\newblock Living on the edge: Phase transitions in convex programs with random
  data.
\newblock \emph{Information and Inference: A Journal of the IMA}, 3\penalty0
  (3):\penalty0 224--294, 2014.

\bibitem[Bartlett and Tewari(2007)]{bartlett2007sparseness}
Peter~L Bartlett and Ambuj Tewari.
\newblock Sparseness vs estimating conditional probabilities: Some asymptotic
  results.
\newblock \emph{Journal of Machine Learning Research}, 8\penalty0
  (Apr):\penalty0 775--790, 2007.

\bibitem[Bartlett et~al.(2020)Bartlett, Long, Lugosi, and Tsigler]{bllt20}
Peter~L. Bartlett, Philip~M. Long, Gábor Lugosi, and Alexander Tsigler.
\newblock Benign overfitting in linear regression.
\newblock \emph{Proceedings of the National Academy of Sciences}, 117\penalty0
  (48):\penalty0 30063–30070, Apr 2020.

\bibitem[Belkin et~al.(2020)Belkin, Hsu, and Xu]{bhx20}
Mikhail Belkin, Daniel Hsu, and Ji~Xu.
\newblock Two models of double descent for weak features.
\newblock \emph{SIAM Journal on Mathematics of Data Science}, 2\penalty0
  (4):\penalty0 1167--1180, 2020.

\bibitem[Berry(1941)]{berry41}
Andrew~C. Berry.
\newblock The accuracy of the gaussian approximation to the sum of independent
  variates.
\newblock \emph{Transactions of the American Mathematical Society}, 49\penalty0
  (1):\penalty0 122--136, 1941.

\bibitem[Buhot and Gordon(2001)]{bg01}
Arnaud Buhot and Mirta~B Gordon.
\newblock Robust learning and generalization with support vector machines.
\newblock \emph{Journal of Physics A: Mathematical and General}, 34\penalty0
  (21):\penalty0 4377--4388, may 2001.

\bibitem[Cand{\`e}s and Sur(2020)]{candes2018phase}
Emmanuel~J Cand{\`e}s and Pragya Sur.
\newblock The phase transition for the existence of the maximum likelihood
  estimate in high-dimensional logistic regression.
\newblock \emph{The Annals of Statistics}, 48\penalty0 (1):\penalty0 27--42,
  2020.

\bibitem[Cao et~al.(2021)Cao, Gu, and Belkin]{cao2021risk}
Yuan Cao, Quanquan Gu, and Mikhail Belkin.
\newblock Risk bounds for over-parameterized maximum margin classification on
  sub-gaussian mixtures.
\newblock \emph{arXiv preprint arXiv:2104.13628}, 2021.

\bibitem[Chatterji and Long(2021)]{cl20}
Niladri~S Chatterji and Philip~M Long.
\newblock Finite-sample analysis of interpolating linear classifiers in the
  overparameterized regime.
\newblock \emph{Journal of Machine Learning Research}, 22\penalty0
  (129):\penalty0 1--30, 2021.

\bibitem[Chen and Donoho(1994)]{chen1994basis}
Shaobing Chen and David Donoho.
\newblock Basis pursuit.
\newblock In \emph{Proceedings of 1994 28th Asilomar Conference on Signals,
  Systems and Computers}, volume~1, pages 41--44. IEEE, 1994.

\bibitem[Cortes and Vapnik(1995)]{cortes1995support}
Corinna Cortes and Vladimir Vapnik.
\newblock Support-vector networks.
\newblock \emph{Machine learning}, 20\penalty0 (3):\penalty0 273--297, 1995.

\bibitem[Cover(1965)]{cover1965geometrical}
Thomas~M Cover.
\newblock Geometrical and statistical properties of systems of linear
  inequalities with applications in pattern recognition.
\newblock \emph{IEEE Transactions on Electronic Computers}, 14\penalty0
  (3):\penalty0 326--334, 1965.

\bibitem[Dietrich et~al.(1999)Dietrich, Opper, and
  Sompolinsky]{dietrich1999statistical}
Rainer Dietrich, Manfred Opper, and Haim Sompolinsky.
\newblock Statistical mechanics of support vector networks.
\newblock \emph{Physical Review Letters}, 82\penalty0 (14):\penalty0 2975,
  1999.

\bibitem[Donoho and Tanner(2009{\natexlab{a}})]{donoho2009counting}
David Donoho and Jared Tanner.
\newblock Counting faces of randomly projected polytopes when the projection
  radically lowers dimension.
\newblock \emph{Journal of the American Mathematical Society}, 22\penalty0
  (1):\penalty0 1--53, 2009{\natexlab{a}}.

\bibitem[Donoho and Tanner(2009{\natexlab{b}})]{donoho2009observed}
David Donoho and Jared Tanner.
\newblock Observed universality of phase transitions in high-dimensional
  geometry, with implications for modern data analysis and signal processing.
\newblock \emph{Philosophical Transactions of the Royal Society A:
  Mathematical, Physical and Engineering Sciences}, 367\penalty0
  (1906):\penalty0 4273--4293, 2009{\natexlab{b}}.

\bibitem[Dvoretsky(1964)]{dvoredsky1961some}
AP~Dvoretsky.
\newblock Some results on convex bodies and banach spaces.
\newblock \emph{Matematika}, 8\penalty0 (1):\penalty0 73--102, 1964.

\bibitem[Fisher and Tippett(1928)]{fisher1928limiting}
Ronald~Aylmer Fisher and Leonard Henry~Caleb Tippett.
\newblock Limiting forms of the frequency distribution of the largest or
  smallest member of a sample.
\newblock In \emph{Mathematical proceedings of the Cambridge philosophical
  society}, volume~24, pages 180--190. Cambridge University Press, 1928.

\bibitem[Germain et~al.(2011)Germain, Lacoste, Laviolette, Marchand, and
  Shanian]{germain2011pac}
Pascal Germain, Alexandre Lacoste, Fran{\c{c}}ois Laviolette, Mario Marchand,
  and Sara Shanian.
\newblock A {PAC-Bayes} sample-compression approach to kernel methods.
\newblock In \emph{ICML}, 2011.

\bibitem[Graepel et~al.(2005)Graepel, Herbrich, and
  Shawe-Taylor]{graepel2005pac}
Thore Graepel, Ralf Herbrich, and John Shawe-Taylor.
\newblock {PAC-Bayesian} compression bounds on the prediction error of learning
  algorithms for classification.
\newblock \emph{Machine Learning}, 59\penalty0 (1-2):\penalty0 55--76, 2005.

\bibitem[Hastie et~al.(2019)Hastie, Montanari, Rosset, and Tibshirani]{hmrt19}
Trevor Hastie, Andrea Montanari, Saharon Rosset, and Ryan~J Tibshirani.
\newblock Surprises in high-dimensional ridgeless least squares interpolation.
\newblock \emph{arXiv preprint arXiv:1903.08560}, 2019.

\bibitem[Hastie and Pregibon(2017)]{hastie2017generalized}
Trevor~J Hastie and Daryl Pregibon.
\newblock Generalized linear models.
\newblock In \emph{Statistical models in S}, pages 195--247. Routledge, 2017.

\bibitem[Hsu et~al.(2021)Hsu, Muthukumar, and Xu]{hmx20}
Daniel Hsu, Vidya Muthukumar, and Ji~Xu.
\newblock On the proliferation of support vectors in high dimensions.
\newblock In \emph{Twenty-Fourth International Conference on Artificial
  Intelligence and Statistics}, 2021.

\bibitem[Huang et~al.(2020)Huang, Hogg, and Villar]{huang2020dimensionality}
Ningyuan Huang, David~W Hogg, and Soledad Villar.
\newblock Dimensionality reduction, regularization, and generalization in
  overparameterized regressions.
\newblock \emph{arXiv preprint arXiv:2011.11477}, 2020.

\bibitem[Ji and Telgarsky(2019)]{ji2019implicit}
Ziwei Ji and Matus Telgarsky.
\newblock The implicit bias of gradient descent on nonseparable data.
\newblock In \emph{Conference on Learning Theory}, pages 1772--1798. PMLR,
  2019.

\bibitem[Ledoux(2001)]{ledoux2001concentration}
Michel Ledoux.
\newblock \emph{The concentration of measure phenomenon}.
\newblock Number~89. American Mathematical Soc., 2001.

\bibitem[Liang and Sur(2020)]{liang2020precise}
Tengyuan Liang and Pragya Sur.
\newblock A precise high-dimensional asymptotic theory for boosting and
  minimum-l1-norm interpolated classifiers.
\newblock \emph{arXiv preprint arXiv:2002.01586}, 2020.

\bibitem[Liu(2019)]{liu2019exact}
Haoyang Liu.
\newblock Exact high-dimensional asymptotics for support vector machine.
\newblock \emph{arXiv preprint arXiv:1905.05125}, 2019.

\bibitem[Mahdaviyeh and Naulet(2019)]{mahdaviyeh2019risk}
Yasaman Mahdaviyeh and Zacharie Naulet.
\newblock Risk of the least squares minimum norm estimator under the spike
  covariance model.
\newblock \emph{arXiv preprint arXiv:1912.13421}, 2019.

\bibitem[Malzahn and Opper(2005)]{malzahn2005statistical}
D{\"o}rthe Malzahn and Manfred Opper.
\newblock A statistical physics approach for the analysis of machine learning
  algorithms on real data.
\newblock \emph{Journal of Statistical Mechanics: Theory and Experiment},
  2005\penalty0 (11):\penalty0 P11001, 2005.

\bibitem[Matousek(2013)]{matousek2013lectures}
Jiri Matousek.
\newblock \emph{Lectures on discrete geometry}, volume 212.
\newblock Springer Science \& Business Media, 2013.

\bibitem[Mei and Montanari(2019)]{mei2019generalization}
Song Mei and Andrea Montanari.
\newblock The generalization error of random features regression: {P}recise
  asymptotics and double descent curve.
\newblock \emph{arXiv preprint arXiv:1908.05355}, 2019.

\bibitem[Mitra(2019)]{mitra2019understanding}
Partha~P Mitra.
\newblock Understanding overfitting peaks in generalization error: Analytical
  risk curves for $\ell_2 $ and $\ell_1 $ penalized interpolation.
\newblock \emph{arXiv preprint arXiv:1906.03667}, 2019.

\bibitem[Muthukumar et~al.(2020)Muthukumar, Vodrahalli, Subramanian, and
  Sahai]{muthukumar2020harmless}
Vidya Muthukumar, Kailas Vodrahalli, Vignesh Subramanian, and Anant Sahai.
\newblock Harmless interpolation of noisy data in regression.
\newblock \emph{IEEE Journal on Selected Areas in Information Theory},
  1\penalty0 (1):\penalty0 67--83, 2020.

\bibitem[Muthukumar et~al.(2021)Muthukumar, Narang, Subramanian, Belkin, Hsu,
  and Sahai]{mnsbhs20}
Vidya Muthukumar, Adhyyan Narang, Vignesh Subramanian, Mikhail Belkin, Daniel
  Hsu, and Anant Sahai.
\newblock Classification vs regression in overparameterized regimes: Does the
  loss function matter?
\newblock \emph{Journal of Machine Learning Research}, 22\penalty0
  (222):\penalty0 1--69, 2021.

\bibitem[Ng(2004)]{ng2004feature}
Andrew~Y Ng.
\newblock Feature selection, l 1 vs. l 2 regularization, and rotational
  invariance.
\newblock In \emph{Proceedings of the twenty-first international conference on
  Machine learning}, page~78, 2004.

\bibitem[Sch{\"o}lkopf et~al.(2000)Sch{\"o}lkopf, Smola, Williamson, and
  Bartlett]{scholkopf2000new}
Bernhard Sch{\"o}lkopf, Alex~J Smola, Robert~C Williamson, and Peter~L
  Bartlett.
\newblock New support vector algorithms.
\newblock \emph{Neural computation}, 12\penalty0 (5):\penalty0 1207--1245,
  2000.

\bibitem[Soudry et~al.(2018)Soudry, Hoffer, Nacson, Gunasekar, and
  Srebro]{shmgs18}
Daniel Soudry, Elad Hoffer, Mor~Shpigel Nacson, Suriya Gunasekar, and Nathan
  Srebro.
\newblock The implicit bias of gradient descent on separable data.
\newblock \emph{J. Mach. Learn. Res.}, 19\penalty0 (1):\penalty0 2822–2878,
  January 2018.

\bibitem[Steinwart(2003)]{steinwart2003sparseness}
Ingo Steinwart.
\newblock Sparseness of support vector machines.
\newblock \emph{Journal of Machine Learning Research}, 4\penalty0
  (Nov):\penalty0 1071--1105, 2003.

\bibitem[Theisen et~al.(2021)Theisen, Klusowski, and Mahoney]{theisen2021good}
Ryan Theisen, Jason Klusowski, and Michael Mahoney.
\newblock Good classifiers are abundant in the interpolating regime.
\newblock In \emph{Proceedings of The 24th International Conference on
  Artificial Intelligence and Statistics}, 2021.

\bibitem[Tibshirani(1996)]{tibshirani1996regression}
Robert Tibshirani.
\newblock Regression shrinkage and selection via the lasso.
\newblock \emph{Journal of the Royal Statistical Society: Series B
  (Methodological)}, 58\penalty0 (1):\penalty0 267--288, 1996.

\bibitem[Vandenberghe(2010)]{vandenberghe2010cvxopt}
Lieven Vandenberghe.
\newblock The cvxopt linear and quadratic cone program solvers, 2010.
\newblock URL
  \url{https://www.seas.ucla.edu/~vandenbe/publications/coneprog.pdf}.

\bibitem[Vapnik(1982)]{vapnik1982estimation}
Vladimir~Naumovich Vapnik.
\newblock \emph{Estimation of dependences based on empirical data}.
\newblock Springer-Verlag, 1982.

\bibitem[Vapnik(1995)]{vapnik1995nature}
Vladimir~Naumovich Vapnik.
\newblock \emph{The Nature of Statistical Learning Theory}.
\newblock Springer-Verlag, 1995.

\bibitem[Vershynin(2015)]{vershynin14}
Roman Vershynin.
\newblock Estimation in high dimensions: a geometric perspective.
\newblock In \emph{Sampling theory, a renaissance}, pages 3--66. Springer,
  2015.

\bibitem[Wang and Thrampoulidis(2021)]{wang2021binary}
Ke~Wang and Christos Thrampoulidis.
\newblock Binary classification of gaussian mixtures: Abundance of support
  vectors, benign overfitting and regularization.
\newblock \emph{arXiv preprint arXiv:2011.09148v4}, 2021.

\bibitem[Wang and Fan(2017)]{wang2017asymptotics}
Weichen Wang and Jianqing Fan.
\newblock Asymptotics of empirical eigenstructure for high dimensional spiked
  covariance.
\newblock \emph{Annals of Statistics}, 45\penalty0 (3):\penalty0 1342, 2017.

\bibitem[Xu and Hsu(2019)]{xu2019number}
Ji~Xu and Daniel Hsu.
\newblock On the number of variables to use in principal component regression.
\newblock In \emph{Advances in Neural Information Processing Systems 32}, 2019.

\end{thebibliography}

\clearpage

\appendix

\section{Proofs for Section~\ref{sec:prelim}}\label{app:prelim-proof}

\newcommand\lesscrazyparagraph[1]{\paragraph{\mbox{#1}}}

We restate and prove Proposition~\ref{prop:equiv}.

\propequiv*

\begin{proof}
  The proof henceforth proceeds under the assumption that $\bK$ is invertible.
  Since $\bK$ is symmetric, the invertibility of $\bK$ implies that all of its principal minors (i.e., the $\bK_{\setminus i}$'s) are invertible.

The equivalences between the first four statements follow from simple implications of the definition of support vector proliferation and the derivation of the dual optimization problems to \eqref{eq:svm-primal-lp} and \eqref{eq:svm-restricted-primal-lp}.
Lemma~1 of \cite{hmx20} proves the equivalence between (1) and (5) for the $\ell_2$ case and completes the argument.
We supplement the argument with an additional equivalence between (4) and (5) to show how the ``leave-one-out terms'' in (5) can be intuitively understood through the geometric framing of the dual problem.

\lesscrazyparagraph{(1) $\iff$ (2):}
This is immediate from the fact that the definition of SVP corresponds exactly to the equality constraints that are present in \eqref{eq:svm-restricted-primal-lp} and not in \eqref{eq:svm-primal-lp}.

\lesscrazyparagraph{(2) $\iff$ (3):}
This equivalence follows by deriving the duals of the two optimization problems, \eqref{eq:svm-restricted-primal-lp} and \eqref{eq:svm-primal-lp}, and by noting that the only difference between the corresponding duals is that the latter has an additional requirement that $\alpha \in\R_+^d$.

By adding Lagrange multipliers $\alpha \in \mathbb{R}^{n}$, we obtain the dual of \eqref{eq:svm-restricted-primal-lp}:
\[
\max_{\alpha \in \mathbb{R}^{n}} \sum_{i=1}^{n}{\alpha_{i}} + \min_{w \in \mathbb{R}^{d}}{  \|w\|_{p} -  w^\T \sum_{i=1}^{n}{\alpha_{i}y_{i}\bx_{i} }}.
\]
By Holder's inequality, if we take $q$ to be the dual of $p$ (i.e. $\frac{1}{p} + \frac{1}{q} = 1$), then $|w^\T u| \le \|w\|_{p} \|u\|_{q}$, with equality when $|u_{i}|^{q}$ is proportional to $|w_{i}|^{p}$.
Therefore,
\[\min_{w \in \mathbb{R}^{d}}  \|w\|_{p} - w^\T u = 
\begin{cases}
0 & \|u\|_{q} \le 1 \\
-\infty & \text{otherwise.}\\
\end{cases}
\]

We further denote $\bA =  \diag(y) \bX  \in \mathbb{R}^{n \times d}$, whose $i$th row is $y_{i}\bx_{i}$. We conclude that the dual of the optimization problem \eqref{eq:svm-restricted-primal-lp} is exactly \eqref{eq:svm-restricted-dual-lp}.
We similarly find that the dual of \eqref{eq:svm-primal-lp} is
\begin{equation*}\label{eq:svm-dual-lp}\tag{SVM Dual}
    \max_{\alpha \in \mathbb{R}^{n}_+} \sum_{i=1}^{n}{\alpha_{i}} \quad
    \text{such that}  \quad \|\bA^\T \alpha \|_{q} \le 1 .
\end{equation*}

Because \eqref{eq:svm-restricted-dual-lp} and \eqref{eq:svm-dual-lp} coincide if and only if the $\alpha$ that solves the former is in the positive orthant ($\alpha \in \R^d_+$), the equivalence follows.

\lesscrazyparagraph{(3) $\iff$ (4):} 
We reconfigure \eqref{eq:svm-restricted-dual-lp} to rewrite the optimization problem as a projection.
Let $\Pi_{\bT}$ and $\Pi_{\bT^+}$ be the $\ell_q$-norm minimizing projection operators onto $\bT$ and $\bT^+$, which are uniquely defined when $\bK$ is invertible. 
Then the solution to \eqref{eq:svm-restricted-dual-lp} is:\footnote{
Note that \eqref{eq:svm-restricted-dual-lp} will never be optimized by $\vec{0}$, because there must always exist some $\alpha$ with strictly positive components such that $\norml[q]{\bA^\T \alpha} \leq 1$. Therefore, we need not worry about the projection being undefined.}
\begin{equation}
    \label{eq:projection-lq}
    \tag{Interpolation Projection}
    \min_{\alpha \in \mathbb{R}^{n}}{ \norm[q]{\bA^\T \frac{\alpha}{\sum_{i=1}^{n}{\alpha_{i}}}} 
    = \norm[q]{\Pi_{\bT}(\vec{0}) }}.
\end{equation}

By the definition of $\bT$ and $\bT^+$ and the fact that $\proj = \bA^\T \alpha^* / \vec{1}^\T \alpha^*$ for $\alpha^*$ optimizing \eqref{eq:projection-lq}, we have that $\alpha^* \in \R_+^d$ if and only if $\proj \in \bT^+$.

\lesscrazyparagraph{(4) $\iff$ (5):}
By the \eqref{eq:projection-lq} formulation, $\proj$ can be alternatively interpreted as the projection of the origin onto the affine space $\mathbf{T}$. 
Therefore we have,
\begin{equation}\label{eq:projection-representation}
    \proj = \sum_{i=1}^{n}{ a^{*}_{i}\bA_{i} } = \arg \min_{u \in \mathbf{T}} \norm[2]{u}^{2}
\end{equation}
where $a^{*}_{i} \in \mathbb{R}$ is proportional to $\alpha_{i}^{*}$ such that $\sum_{i=1}^{n}{a^{*}_{i}} = 1$. 
This is possible because the optimal value of \eqref{eq:svm-restricted-primal-lp} is positive. 

The following steps show the equivalence:
\begin{enumerate}
    \item For every $i \in [n]$ we show,
    \begin{equation}\label{eq:coef-interpretation}
        a^{*}_{i} > 0 \quad \Longleftrightarrow \quad \bA_{i}^\T \projj{i}  < \norm[2]{\projj{i}}^{2}
    \end{equation}
    by leveraging the fact that $\ell_{2}$ space is equipped with inner product which allows us to decompose the contribution of each sample to the projection. 
    
    \item We find an explicit expression for $\proj$:
        \begin{equation}\label{eq:projection-explicit}
            \proj =  \frac{\bA^\T \paren{\bA\bA^\T}^{-1} \one }{\one^\T \paren{\bA\bA^\T}^{-1} \one }= \frac{\bX \paren{\bX^{\T}\bX}^{-1} y }{y^\T \paren{\bX^{\T}\bX}^{-1} y }.
        \end{equation}
        An analogous expression for $\projj{i}$ can be found for any fixed $i$ by the same method, which gives the desired equivalence by combining with \eqref{eq:coef-interpretation}.
\end{enumerate}

\emph{First Step:} Fix some index $i \in [n]$. Since $\Pi_{\mathbf{T}_{\setminus i}}(\bA_i)$ is on the affine space $\mathbf{T}_{\setminus i}$, which is closed under affine linear combination, one can express $\mathbf{T}$ in the following way:
\begin{align*}
    \mathbf{T} &= \Big\{ \sum_{j=1}^{n}{a_{j}\bA_j} : \sum_{j=1}^{n}{a_{j}} = 1 \Big\} \\
    &= \Big\{ a_{i} \paren{\bA_i - \Pi_{\mathbf{T}_{\setminus i}}(\bA_i)} + \Big(a_{i}  \Pi_{\mathbf{T}_{\setminus i}}(\bA_i) + \sum_{j \neq i}{a_{j}\bA_{j}}\big)  : \sum_{j \neq i}{a_{j}} = 1-a_{i} \ , a_{i}\in \mathbb{R} \Big\} \\
    &= \Big\{ a_{i} (\bA_i - \Pi_{\mathbf{T}_{\setminus i}}(\bA_i)) + u: u \in \mathbf{T}_{\setminus i} \ , a_{i}\in \mathbb{R} \Big\} .
\end{align*}

By the definition of the projection onto $\mathbf{T}$, we represent $\proj = a^{*}\paren{\bA_i - \Pi_{\mathbf{T}_{\setminus i}}(\bA_i)} + u^{*}$ where,
\[
(a^{*}, u^{*}) = \argmin_{(a, u) \in \mathbb{R} \times \mathbf{T}_{\setminus i}}  \norm[2]{ a( \bA_{i} - \Pi_{\mathbf{T}_{\setminus i}}(\bA_{i}) ) + u }^{2}.
\]
It is straightforward to see that $a^{*} = a_{i}^{*}$ by comparing this representation with equation \eqref{eq:projection-representation} alongside with the fact that $\bA_i - \Pi_{\mathbf{T}_{\setminus i}}(\bA_i)$ is orthogonal to $\mathbf{T}_{\setminus i}$:

\[
\zero = \proj - \proj = \paren{a^{*} - a_{i}^{*}}\paren{\bA_i - \Pi_{\mathbf{T}_{\setminus i}}(\bA_i)} + \Big( u^{*} - \underbrace{\Big(\sum_{j \neq i}{a_{j}^{*} \bA_{j}} + a_{i}^{*} \projj{i}\Big)}_{ \in \mathbf{T}_{\setminus i}}  \Big).
\]

To find $a^{*}$, we sequentially optimize over $u$, substitute its optimal value, and then optimize over $a$. 
It suffices to minimize $\norm[2]{u}^{2}$ because $ u^\T (\bA_i - \Pi_{\mathbf{T}_{\setminus i}}(\bA_i) )$ is constant for all $u \in \mathbf{T}_{\setminus i}$ due to orthogonality and the definition of $\mathbf{T}_{\setminus i}$.
Hence, $u^{*} = \projj{i}$ is optimal. 
Subsequently, by optimizing over $a$ and setting the derivative to zero, 
\[
( \bA_{i} - \Pi_{\mathbf{T}_{/i}}(\bA_{i}))^\T  \paren{a^{*} \paren{ \bA_{i} - \Pi_{\mathbf{T}_{/i}}(\bA_{i})} + \projj{i}} = 0
\]
or equivalently,
\[
a^{*} 
= \frac{(\Pi_{\mathbf{T}_{/i}}(\bA_{i}) - \bA_{i})^\T \projj{i}}{ \norm[2]{\bA_{i} - \Pi_{\mathbf{T}_{\setminus i}}(\bA_{i})}^{2} } 
= \frac{ \norm[2]{ \projj{i} }^{2} -  \bA_{i}^\T \projj{i} }{ \norm[2]{\bA_{i} - \Pi_{\mathbf{T}_{/i}}(\bA_{i})}^{2} }
\]
For the last step, we combine the facts that $ \projj{i}^\T u$ is constant over $\mathbf{T}_{\setminus i}$ and $\projj{i} \in \mathbf{T}_{\setminus i}$. 
Note that the denominator is non-zero because the invertibility of $\bK$ implies that $\bA_i$ will not lie on the span of the remaining rows $\bA_{\setminus i}$, and hence will not be in $\bT_{\setminus i}$.
This immediately proves \eqref{eq:coef-interpretation}. 

\emph{Second Step:} We now shift our focus to expressing $\proj$ explicitly in terms of $\bA$. As discussed in Section \ref{sec:prelim}, the projection of the origin onto $\mathbf{T}$ for $\ell_{2}$-SVM can be expressed in an explicit way as $\proj = \bA^\T \fracl {\alpha^{*}}{ \one^{\T} \alpha^{*} }$, where $\alpha^{*}$ is the unique solution to \eqref{eq:svm-restricted-dual-lp} for $p=q=2$. In order to represent $\alpha^{*}$ explicitly one can reform \eqref{eq:svm-restricted-dual-lp} using a simple change of variables $\nu = \paren{\bA\bA^\T}^{\fracl{1}{2}}\alpha$ into,
\[
\nu^{*} = \argmax_{\nu \in \R^n} \nu^\T \paren{\bA\bA^\T}^{\fracl{-1}{2}} \one \quad \text{s.t. } \norm[2]{\nu} \le 1.
\]
This problem can be easily understood using a simple Cauchy-Schwarz inequality, 
\[
\alpha^{*} = \paren{\bA\bA^\T}^{\fracl{-1}{2}} \nu^{*} = \fracl{\paren{\bA\bA^\T}^{-1} \one}{ \sqrt{ \one^\T \paren{\bA\bA^\T}^{-1} \one } .}
\]
In conclusion, we can express the projection as,
\[
\proj = \frac{\bA^\T \alpha^{*}}{\one^\T \alpha^* } = \frac{\bA^\T \paren{\bA\bA^\T}^{-1} \one }{\one^\T \paren{\bA\bA^\T}^{-1} \one } = \frac{\bX^\T \paren{\bX\bX^\T}^{-1} y }{y^\T \paren{\bX\bX^\T}^{-1} y }.\qedhere
\]
\end{proof}

\section{Proofs for Section~\ref{sec:asymptotic}}\label{app:lemmas}

We prove Theorem~\ref{thm:anisotropic}, which we restate below.

\thmanisotropic*

\begin{proof}

As discussed in Section~\ref{sec:nonasymp}, it suffices to prove that $\bK_{\setminus i}$ is invertible for all $i$ and
\[
\max_{i \in [m]} \bracket{y_i y_{\setminus i}^\T \paren{\bK_{\setminus i}^{-1}-\frac{1}{\norm[1]{\lambda}} I_{n-1}} \bX_{\setminus i} \bx_i 
+ \frac{1}{\norm[1]{\lambda}}y_i y_{[m] \setminus i}^\T \bX_{[m] \setminus i} \bx_i 
+ \frac{1}{\norm[1]{\lambda}}y_i y_{\setminus [m]}^\T \bX_{\setminus [m]} \bx_i} \geq 1
\]
with probability $1-\delta$ for some $m \leq n$.
We do so by showing that the following three events each hold with probability $1 - \frac{\delta}{3}$:
\begin{align*}
    &\max_{i \in [m]} \abs{y_i y_{\setminus i}^\T \paren{\bK_{\setminus i}^{-1}-\frac{1}{\norm[1]{\lambda}} I_{n-1}} \bX_{\setminus i} \bx_i} \leq 1, \\
    &\max_{i \in [m]} \frac{1}{\norm[1]{\lambda}} \abs{y_i y_{[m] \setminus i}^\T \bX_{[m] \setminus i} \bx_i} \leq 1,  \ \text{and} \\
    &\max_{i \in [m]} \frac{1}{\norm[1]{\lambda}}y_i y_{\setminus [m]}^\T \bX_{\setminus [m]} \bx_i \geq 3.
\end{align*}

It remains to plug in the results of Lemmas~\ref{lemma:e1}, \ref{lemma:e2}, and \ref{lemma:e3} to show that the three events occur with high probability given the conditions imposed on $n$, $d_2$, $d_{\infty}$, and $\delta$ in \eqref{eq:3}.
Let $m := \ceill{\exp(\frac{d_2}{2\Cb n})}$. 
By \eqref{eq:3}, $m \leq \sqrt{n}+1 \leq \frac{n}{\ln n} \leq \frac{n}{2}$ for sufficiently small constant $\Ca$.


\begin{enumerate}
    \item 
    By Lemma~\ref{lemma:e1} in Appendix~\ref{asec:e1} with $\delta := \frac{\delta}{3m}$, it follows that for any fixed $i \in [m]$, $\bK_{\setminus i}$ is invertible and
    \[\abs{y_i y_{\setminus i}^\T \paren{\bK_{\setminus i}^{-1} - \frac{1}{\norml[1]{\lambda}} I_{n-1} } \bX_{\setminus i} \bx_i } \leq 1\] with probability at least $1 - \frac{\delta}{3m}$ as long as the following conditions hold:
    \begin{align}
      d_{\infty} & \geq \cc'\paren{n\paren{\log\frac{3m}{\delta}}^{1/3} + n^{1/3} \log \frac{3m}{\delta}} ,
      \label{eq:proof-anisotropic-cond1}
      \\
      d_2 d_{\infty} & \geq \cd n \log \frac{3m}{\delta} \paren{n + \log \frac{3m}{\delta}}.
      \label{eq:proof-anisotropic-cond2}
    \end{align}
    We show that the inequalities in
    \eqref{eq:proof-anisotropic-cond1}
    and
    \eqref{eq:proof-anisotropic-cond2}
    are implied by the preconditions of Theorem~\ref{thm:anisotropic} in
    \eqref{eq:3} by choosing sufficiently large constants $C_1$, $C_3$, and $C_4$.
    For \eqref{eq:proof-anisotropic-cond1},
    \begin{align*}
        n \paren{\log \frac{3m}{\delta}}^{1/3} + n^{1/3} \log \frac{3m}{\delta} 
        &\leq n \paren{\frac{d_2}{C_2 n} + \log \frac{3}{\delta}}^{1/3} + n^{1/3}\log \frac{3n}{2\delta} \\
        &\leq \frac{n^{2/3}d_2^{1/3}}{C_2}  + n \paren{\log\frac{3}{\delta}}^{1/3}  + n \log \frac{3}{\delta} \\
        &\leq \frac{n^{1/3} d_\infty^{2/3}}{C_2 C_4^{1/3}}+ 2 n \log \frac{3}{\delta} ,
    \end{align*}
    which implies the desired inequality.
    To establish \eqref{eq:proof-anisotropic-cond2}:
%
    \begin{align*}
        n \log \frac{3m}{\delta} \paren{n + \log \frac{3m}{\delta}}
        &\leq \frac{d_2}{C_2} \paren{n + \log \frac{3n}{2\delta}} 
        \leq  \frac{d_2}{C_2} \paren{2n + \log \frac{1}{\delta}} \leq \frac{d_2}{C_2} \cdot \frac{3d_\infty}{C_3}.
    \end{align*}
    
    By applying a union bound to all $m$ events, they all occur with probability at least $1 - \frac{\delta}{3}$.

	\item 
	By applying Lemma~\ref{lemma:e2} in Appendix~\ref{asec:e2} for all $i \in [m]$ with $\delta$ as before and union-bounding over the corresponding events, with probability $1  - \frac{\delta}{3}$, \[\frac{1}{\norml[1]{\lambda}}   y_i y_{[m] \setminus i}^\T \bX_{[m] \setminus i} \bx_i \geq -1\] for all $i \in [m]$ as long as
	 \[d_2 \geq \ca m \log \frac{3m}{\delta}
	 \quad \text{and} \quad 
	 d_\infty \geq \cb \sqrt{m} \log \frac{3m}{\delta}.
    \]
    Both inequalities follow from the third inequality of \eqref{eq:3} for sufficiently large $C_3$ and by $m \leq \frac{n}{\log n}$.

	\item By Lemma~\ref{lemma:e3} in Appendix~\ref{asec:e3} with $t := 3$, 
	\[\max_{i \in [m]} \frac{1}{\norml[1]{\lambda}} y_i y_{\setminus [m]}^\T \bX_{ \setminus [m]} \bx_i \geq 3\]
	with probability $1 - \frac{\delta}{3}$, if
	\begin{align*}
	     & n-m \geq \ca \paren{\log \frac{3}{\delta}}^2, &&
     \exp\paren{\frac{9 d_2}{\cb (n-m)}} \leq m \leq (n-m), \\
    & d_2 \geq \cc (n-m) \log \log \frac{3}{\delta},
    && \text{and} \quad
    d_{\infty}^2 \geq\cd d_{2} (n-m).
    \end{align*}
    The inequalities are satisfied as immediate consequences of \eqref{eq:3} and the fact that $m = \ceill{\exp(\frac{d_2}{2C_2n})} \leq \frac{n}{2}$ for sufficiently small $C_2$. \qedhere
\end{enumerate}
\end{proof}

In the subsequent three sections, we prove Lemmas~\ref{lemma:e1}, \ref{lemma:e2}, and \ref{lemma:e3}.

\subsection{Bounded difference between leave-one-out terms with $\bK^{-1}$ and scaled identity}\label{asec:e1}

\begin{lemma}\label{lemma:e1}
    Let $(\bX, y) \in \R^{n \times d} \times \R^n$ and $(\bx', y') \in \R^{d} \times \R$ be $\lambda$-anisotropic subgaussian samples that are independent of one another with $\bK = \bX \bX^\T$.
    Pick any $\delta \in (0, \frac{1}{2})$.
    There exist universal constants $\cc, \ccc$ such that if $d_{\infty} \geq \cc(n + \log \frac{1}{\delta})$,
    then with probability $1 - \delta$, $\bK$ is invertible and
    \[\abs{y^\T \paren{\bK^{-1} - \frac{1}{\norm[1]{\lambda}} I_{n}} \bX \bx' } \leq c\sqrt{\frac{n \log \frac{1}{\delta}}{d_{\infty}}}\paren{\sqrt{\frac{n + \log \frac{1}{\delta}}{d_2}} + \frac{n + \log \frac{1}{\delta}}{d_\infty}}.\]
\end{lemma}
\begin{remark}
    To guarantee that $\absl{y^\T \parenl{\bK^{-1} - \norm[1]{\lambda}^{-1} I_{n}} \bX \bx' } \leq \eps$
    with probability $1 - \delta$ for some $\eps > 0$,
    it suffices to show that 
    \[d_\infty \geq \cc' \cdot \frac{n \paren{\log \frac{1}{\delta}}^{1/3} + n^{1/3} \log \frac{1}{\delta}}{\eps^{2/3}} 
    \quad \text{and} \quad
    d_2 d_\infty \geq \cd \cdot \frac{n \log \frac{1}{\delta} \paren{n + \log \frac{1}{\delta}}}{\eps^2},\]
    for universal constants $\cc'$ and $\cd$.
\end{remark}

The proof of Lemma~\ref{lemma:e1} relies heavily on a concentration bound on the eigenvalues of the Gram matrix $\bK$, which draws from a technical lemma of \citet{hmx20}. 
We present and prove this result below and then use it to prove Lemma~\ref{lemma:e1}.

\begin{lemma}\label{lemma:gram-evals}
	Let $(\bX, y) \in \R^{n \times d} \times \R$ be $\lambda$-anisotropic subgaussian samples with Gram matrix $\bK := \bX \bX^\T \in \R^{n \times n}$.
	Pick any $\delta \in (0, \half)$.
	For some universal constant $\ccc$, with probability $1-\delta$,
	\begin{equation}
	\label{eq:K-norm-bound}
	    \normop{\bK - \norm[1]{\lambda} I_n} \leq \ccc \norm[1]{\lambda} \paren{\sqrt{\frac{n + \log \frac{1}{\delta}}{d_2}} + \frac{n + \log \frac{1}{\delta}}{d_\infty}} ,
	\end{equation}
  where $\normop{\cdot}$ denotes the spectral (operator) norm.
	If additionally $d_\infty \geq \cc (n + \log \frac{1}{\delta})$ for some universal constant $\cc$, then for the same event, $\bK$ is invertible and
	\begin{equation}
	\label{eq:Kinv-norm-bound}
	\norm{\bK^{-1} - \frac{1}{\norm[1]{\lambda}} I_n} \leq  \frac{\ccc}{\norm[1]{\lambda}} \paren{\sqrt{\frac{n + \log \frac{1}{\delta}}{d_2}} + \frac{n + \log \frac{1}{\delta}}{d_\infty}} .
	\end{equation}
\end{lemma}
\begin{proof}

Equation~\eqref{eq:K-norm-bound} follows from Lemma~8 of \citet{hmx20}.
For some universal constant $\cp$ and sufficiently large $\ccc$, we have the following:
\begin{align*}
    &\pr{ \normop{\bK - \norm[1]{\lambda} I_n} \geq \ccc \norm[1]{\lambda} \paren{\sqrt{\frac{n + \log \frac{1}{\delta}}{d_2}} + \frac{n + \log \frac{1}{\delta}}{d_\infty}}}\\
    &\quad\leq 2 \cdot 9^n \cdot \exp\paren{-\cp\min\paren{\frac{\ccc^2\norm[1]{\lambda}^2(n + \log \frac{1}{\delta})}{\norm[2]{\lambda}^2 d_2},\frac{\ccc\norm[1]{\lambda}(n + \log \frac{1}{\delta})}{\norm[\infty]{\lambda} d_{\infty}}}}\\
    &\quad \leq 2 \exp\paren{n\log 9 - \cp \min(\ccc^2, \ccc) \paren{n + \log \frac{1}{\delta}}} \leq \delta.
\end{align*}

If equation~\eqref{eq:Kinv-norm-bound} holds and $\cc$ is sufficiently large, then all eigenvalues of $\bK$ are strictly positive and $\bK$ is invertible.
We now derive equation~\eqref{eq:Kinv-norm-bound} by bounding the eigenvalues of $\bK^{-1}$, assuming that the event in equation~\eqref{eq:K-norm-bound} occurs, and rescaling $\ccc$:
\begin{align*}
\lefteqn{
    \normop{\bK^{-1} - \frac{1}{\norm[1]{\lambda}} I_n}
    } \\
    &\leq \max\paren{\mu_{\max}(\bK^{-1}) - \frac{1}{\norm[1]{\lambda}}, - \mu_{\min}(\bK^{-1}) + \frac{1}{\norm[1]{\lambda}}} \\
    &= \max\paren{\frac{1}{\mu_{\min}(\bK)} - \frac{1}{\norm[1]{\lambda}}, - \frac{1}{\mu_{\max}(\bK)} + \frac{1}{\norm[1]{\lambda}}} \\
    &\leq \frac{1}{\norm[1]{\lambda}} \max\paren{\frac{1}{1 - \ccc  \paren{\sqrt{\frac{n + \log \frac{1}{\delta}}{d_2}} + \frac{n + \log \frac{1}{\delta}}{d_\infty}}} - 1, 1- \frac{1}{1 + \ccc  \paren{\sqrt{\frac{n + \log \frac{1}{\delta}}{d_2}} + \frac{n + \log \frac{1}{\delta}}{d_\infty}}}} \\
    &\leq \frac{1}{\norm[1]{\lambda}}\cdot 2\ccc\paren{\sqrt{\frac{n + \log \frac{1}{\delta}}{d_2}} + \frac{n + \log \frac{1}{\delta}}{d_\infty}}.\qedhere
\end{align*}
\end{proof}

\begin{proof}[Proof of Lemma~\ref{lemma:e1}]
Conditioned on $\bX$, $y^\T (\bK^{-1} - \norml[1]{\lambda}^{-1} I_{n}) \bX \bx'$ is a univariate subgaussian random variable with mean 0 and variance proxy at most $\norml{y^\T (\bK^{-1} - \norml[1]{\lambda}^{-1} I_{n}) \bX}^2 \norml[\infty]{\lambda}$, as long as $\bK$ is invertible.
We bound the variance proxy and show that $\bK$ is invertible with high probability by applying Lemma~\ref{lemma:gram-evals} for some universal $\cp$ with probability $1 - \frac{\delta}{2}$:
\begin{align*}
    \norm{y^\T (\bK^{-1} - \frac{1}{\norm[1]{\lambda}} I_{n}) \bX}^2 \norm[\infty]{\lambda}
    &\leq \norm[2]{y}^2 \norm{\bK^{-1} - \frac{1}{\norm[1]{\lambda}} I_{n}}^2 \norm{\bK} \norm[\infty]{\lambda} \\
    &\leq n \cdot \frac{\cp}{\norm[1]{\lambda}^2} \paren{\sqrt{\frac{n + \log \frac{1}{\delta}}{d_2}} + \frac{n + \log \frac{1}{\delta}}{d_\infty}}^2 \cdot 2 \norm[1]{\lambda} \norm[\infty]{\lambda} \\
    &=  \frac{2\cp n}{d_\infty} \paren{\sqrt{\frac{n + \log \frac{1}{\delta}}{d_2}} + \frac{n + \log \frac{1}{\delta}}{d_\infty}}^2 
\end{align*}
    We observe that the bound holds for a proper choice of $\ccc$ for a standard concentration bound for a subgaussian random variable.
\end{proof}

\subsection{Concentration of leave-one-out terms}\label{asec:e2}

\begin{lemma}\label{lemma:e2}
    Let $(\bX, \bZ, y) \in \R^{n \times d} \times \R^{n \times d} \times \R^n$ and $(\bx', \bz', y') \in \R^d \times \R^{d} \times \R$ be $\lambda$-anisotropic subgaussian samples that are independent of one another.
    Pick any $\delta \in (0, \half)$.
    There exists a universal constant $\ccc$ such that with probability $1 - \delta$,
    \[\frac{1}{\norm[1]{\lambda}} \abs{y^\T \bX \bx'} \leq c \paren{\sqrt{\frac{n \log \frac{1}{\delta}}{d_2}} + \frac{\sqrt{n} \log \frac{1}{\delta}}{d_\infty}}.\]
\end{lemma}
\begin{remark}
    To ensure that $\norm[1]{\lambda}^{-1} \absl{y^\T \bX \bx'} \leq \eps$ with probability $1 - \delta$ for some $\eps > 0$, it suffices to show that
    \[d_2 \geq \frac{\ca n \log \frac{1}{\delta}}{\eps^2}
    \quad \text{and} \quad
    d_\infty \geq \frac{\cb \sqrt{n} \log \frac{1}{\delta}}{\eps}\]
    for universal constants $\ca, \cb$.
\end{remark}



\begin{proof}
    We can rewrite $y^\T \bX \bx'$ for some vector of 1-subgaussian random variables $\tilde{\bz} \in \R^d$.
    \begin{align*}
        y^\T \bX \bx'
        &= y^\T \bZ \diag(\lambda) \bz'
        = \sqrt{n} \tilde{\bz}^\T \diag(\lambda) \bz'
        = \sqrt{n} \sum_{i=1}^d \lambda_i \tilde{\bz}_i \bz'_i.
    \end{align*}
    By Lemma~2.7.7 of \cite{vershynin14}, each $\tilde{\bz}_i \bz'_i$ is an independent $(1,2)$-subexponential random variable.
    Thus, $\sum_{i=1}^d \lambda_i \tilde{\bz}_i \bz'_i$ is $( \norml[2]{\lambda}^2, 2 \norml[\infty]{\lambda})$-subexponential, and with probability $1 - \delta$,
    \[\abs{ y^\T \bX \bx'} \leq \sqrt{n} \ccc \paren{\sqrt{\norm[2]{\lambda}^2 \log \frac{1}{\delta}} + \norm[\infty]{\lambda} \log \frac{1}{\delta}}.\]
    We have the claim by dividing by $\norml[1]{\lambda}$.
\end{proof}

\subsection{Anti-concentration for independent leave-one-out terms}\label{asec:e3}

\begin{lemma}\label{lemma:e3}
  Let $(\bX,\bZ,y) \in \R^{n \times d} \times \R^{n \times d} \times \R^n$
  and
  $(\bX',\bZ',y') \in \R^{m \times d} \times \R^{m \times d} \times \R^m$
  be $\lambda$-anisotropic subgaussian samples that are independent of each other.
  Pick any $t>0$ and $\delta \in (0, \half)$. 
    For universal constants $\ca$, $\cb$, $\cc$, and $\cd$, if 
    \[
    n \geq \ca \paren{\log \frac{1}{\delta}}^2,
    \quad
     \exp\paren{\frac{t^2 d_2}{\cb n}} \leq m \leq n,
    \quad
    d_2 \geq \cc n \log \log \frac{1}{\delta},
    \quad \text{and} \quad
    d_{\infty}^2 \geq\cd d_{2} n,\]
    then with probability $1 - \delta$
    \[\max_{i \in [m]} \frac{1}{\norm[1]{\lambda}}  y^\T \bX \bx_i' \geq t.\]
\end{lemma}

\begin{proof}
\newcommand{\bxig}{\text{$\bX$ is good}}
\newcommand{\bxing}{\text{$\bX$ is not good}}

Let $\bq_{j} := \norml[1]{\lambda}^{-1} y^\T \bZ_{\cdot, j}$ for $j \in [d]$, where $\bZ_{\cdot, j} = (\bz_{1,j}, \dots, \bz_{n,j}) \in \R^n$.
Note that 
\[\frac{1}{\norm[1]{\lambda}}  y^\T \bX \bx_i' = \sum_{j=1}^d  \bq_{j} \lambda_j  \bz_{i,j}'\] and all $\bq_{j}$ and $\bz_{i,j}'$ are independent and that $\bq_j$ is subgaussian with variance proxy $\fracl{n}{\norml[1]{\lambda}^2}$.
The proof of the claim is powered by the Berry-Esseen theorem and comes in two parts:
\begin{itemize}
    \item We first show that $\bX$ is well-behaved with high probability; we say that $\bX$ is \emph{good} if the following hold:
    \begin{equation}\label{eq:good-norm}
         \sum_{j=1}^d \lambda_j^2 \bq_j^2 \geq \frac{n}{2 d_2},
    \end{equation}
    \begin{equation}\label{eq:good-max}
         \max_j \abs{\lambda_j \bq_j} \leq \frac{n^{1/4}}{\sqrt{d_2}}.
    \end{equation}
    We prove that $\prl{\bxig} \geq 1 - \frac{2\delta}{3}$.

    \item 
    Then, we use the Berry-Esseen Theorem to show that, for each $i \in [m]$,
    \[\pr{\frac{1}{\norm[1]{\lambda}} y^{\T} \bX \bx_i' \leq t \mid \bxig } \leq \paren{\frac{\delta}{3}}^{1/m} . \] 
\end{itemize}
Because $\norm[1]{\lambda}^{-1} y_i' y^\T \bX \bx_i'$ for $i \in [m]$ are conditionally independent given $\bX$, we obtain the desired statement:
\begin{align*}
    \pr{\max_{i \in [m]} \frac{1}{\norm[1]{\lambda}} y_i' y^\T \bX \bx_i' \leq t}
    &\leq \pr{\bxing} + \pr{\max_{i \in [m]} \frac{1}{\norm[1]{\lambda}} y_i' y^\T \bX \bx_i' \leq t \mid \bxig} \\
    & \leq \frac{2\delta}{3} + \prod_{i=1}^m \pr{ \frac{1}{\norm[1]{\lambda}} y_i' y^\T \bX \bx_i' \leq t \mid \bxig} \leq \delta .
\end{align*}

\paragraph{$\bX$ satisfies \eqref{eq:good-norm} with high probability.}
The claim follows from the Hanson-Wright inequality \cite{vershynin14}, the lower-bound on $n$ with respect to $\delta$, the assumption $d_\infty^2 \geq c_4 d_2 n$, and the fact $\norm[4]{\lambda}^4 \leq \norm[2]{\lambda}^2 \norm[\infty]{\lambda}^2$.
\begin{align*}
    \pr{\sum_{j=1}^d \lambda_j^2 \bq_j^2 \leq \frac{n}{2d_2}}
    &\leq \pr{\abs{\sum_{j=1}^d \lambda_j^2 \bq_j^2 - \EE{\sum_{j=1}^d \lambda_j^2 \bq_j^2}} \leq \frac{n}{2d_2}} \\
    &\leq 2 \exp\paren{-c'' \min\paren{\frac{n^2}{4d_2^2} \cdot \frac{\norm[1]{\lambda}^4}{n^2} \cdot \frac{1}{\norm[4]{\lambda}^4}, \ \frac{n}{2d_2} \cdot \frac{\norm[1]{\lambda}^2}{n} \cdot \frac{1}{\norm[\infty]{\lambda}^2}}} \\
    &\leq 2\exp\paren{- \frac{c'' d_\infty^2}{4d_2}} 
    \leq 2\exp\paren{- \frac{c'' c_4 n}{4}}
    \leq \frac{\delta}{3},
\end{align*}
for sufficiently large absolute constant $c_4$.

\paragraph{$\bX$ satisfies \eqref{eq:good-max} with high probability.}

We introduce a different sequence of random variables $\br_1, \dots, \br_{n}$ to eliminate any dependence on $d$.
Set $0 = k_0 < k_1 < \dots < k_{n} = d$ such that for all $i \in [n]$, 
\[\sum_{j=k_{i-1}+1}^{k_{i}} \lambda_j^2 \leq \frac{2\norm[2]{\lambda}^2}{n}.\]
Such a partition is possible because we can require that $\norml[\infty]{\lambda}^2 \leq \fracl{\norml[2]{\lambda}^2}{n}$ by assuming $\cd$ is large enough.
Let \[\br_i = \sum_{j=k_{i-1}+1}^{k_{i}} \lambda_j^2 \bq_j^2.\]  
Note that all $\br_i$ are independent and that \[\sum_{j=1}^d \lambda_j^2 \bq_j^2 = \sum_{i=1}^{n} \br_i.\] 
Because $\EEl{\bq_j} = 0$ and $\EEl{\bq_j^2} = \fracl{n}{\norml[1]{\lambda}^2}$, we can bound $\EEl{\br_i}$:
\[	\EE{\br_i} 
		\leq \frac{2n\norm[2]{\lambda}^2}{ n \norm[1]{\lambda}^2}
		\leq \frac{2}{d_2}. \]
		
We upper-bound $\max_j \absl{\lambda_j \bq_j}$ by noting that $\max_j \absl{\lambda_j \bq_j} \leq \max_i \sqrt{\br_i}$.
By the Hanson-Wright inequality, the subgaussianity of $\bq_i$, and the lower-bound on $n$ with respect to $\delta$,
\begin{align*}
    \pr{\br_i \geq \frac{ \sqrt{n}}{d_2}}
    &\leq \pr{\br_i \geq \EE{\br_i} + \frac{\sqrt{n}}{2 d_2}} \\
    &\leq 2\exp\paren{-c' \min\paren{\frac{ n}{4d_2^2}\cdot \frac{ \norm[1]{\lambda}^4}{ n^2} \cdot \frac{1}{\sum_{j=k_{i-1}+1}^{k_i} \lambda_j^4}, \frac{ \sqrt{n}}{2 d_2}\frac{ \norm[1]{\lambda}^2}{n} \frac{1}{\max_{k_{i-1} < j \leq k_i} \lambda_j^2}}} \\
    &\leq 2\exp\paren{-c' \min \paren{ \frac{n}{4}, \frac{\sqrt{n}}{2}}} 
    \leq \frac{\delta}{3n}
\end{align*}
for some absolute constants $c'$ and for a sufficiently large setting of $c_1$.
Thus, \eqref{eq:good-max} is satisfied with probability $1 - \frac{\delta}{3}$ by a union bound.

\paragraph{Bound on term given good $\bX$.}
We use the Berry-Esseen theorem \cite{berry41} to relate the maximization over $\norml[1]{\lambda}^{-1} y^\T \bX \bx_1'$ to a maximization over standard Gaussians.
Consider some fixed good $\bX$ (and hence, fixed $\bq_j$ for all $j \in [d]$). Then, for some absolute constant $\ccc$ and for univariate standard Gaussian $\bg$,
\[\sup_{t \in \R} \abs{\pr{\frac{\norm[1]{\lambda}^{-1} y^\T \bX \bx_1'}{\sqrt{\sum_{j=1}^d \bq_j^2\lambda_j^2 \EE{ \bz_{i,j}^{\prime2}}}} \leq t} - \pr{\bg \leq t}} 
\leq \frac{\ccc}{\sqrt{\sum_{j=1}^d \bq_j^2\lambda_j^2 \EE{ \bz_{i,j}^{\prime2}}}} \cdot \max_{j \in [d] }\frac{\abs{\bq_{j}^3}\lambda_j^3\EE{\abs{\bz_{i,j}^{\prime3}} }}{\bq_{j}^2\lambda_j^2\EE{\bz_{i,j}^{\prime2} }}.\]

Because each $\bz_{i,j}$ is subgaussian, $\EEl{\bz_{i,j}^3} \leq \rho = O(1)$ for some $\rho$. 
We simplify the expression by plugging in the second and third moments of $\bz_{i,j}'$ and rescaling $t$:
\[ \sup_{t \in \R}\abs{\pr{\frac{1}{\norm[1]{\lambda}} y^\T \bX \bx_1' \leq t} - \pr{\bg \leq \frac{t}{\sqrt{\sum_{j=1}^d \lambda_j^2 \bq_j^2}}}} \leq \frac{\ccc \rho}{\sqrt{\sum_{j=1}^d \lambda_j^2\bq_j^2  }} \cdot \max_{j \in [d] } \abs{\bq_j\lambda_j}.\]

Because we assume that $\bX$ is good, we plug in our upper-bound on $\max_j \absl{\lambda_j \bq_j}$ and lower-bound on $\sum_{j=1}^d \lambda_j^2 \bq_j^2$.
\begin{align*}
	\pr{\frac{1}{\norm[1]{\lambda}} y^\T \bX \bx_1' \leq t}
	&\leq \pr{\bg \leq \frac{t}{\sqrt{\sum_{j=1}^d \lambda_j^2 \bq_j^2}}} + \frac{\ccc \rho}{\sqrt{\sum_{j=1}^d \lambda_j^2 \bq_j^2 }} \cdot \max_{j \in [d] } \abs{\lambda_j \bq_j} \\
	&\leq \pr{\bg \leq \frac{t\sqrt{2d_2}}{\sqrt{n}}} + \frac{ \ccc \rho \sqrt{2d_2}}{\sqrt{n}} \cdot   \frac{n^{1/4}}{\sqrt{d_2}} \\
	&\leq \pr{\bg \leq \frac{t\sqrt{2d_2}}{\sqrt{n}}} + \frac{\sqrt{2}c \rho}{n^{1/4}}
\end{align*}
	
We now bound the first term by invoking the
Mills ratio bound for the Gaussian distribution function (Fact~\ref{fact:mills})
with the assumption that $\cb \leq \frac{1}{8}$.
The remainder of the inequalities follow by enforcing that $\ca$ and $\cb$ be sufficiently large and small respectively:
\begin{align*}
	\pr{\bg \leq \frac{t\sqrt{2d_2}}{\sqrt{n}}}
	&\leq \pr{\bg \leq \sqrt{\frac{1}{4}\log m}}\\
	& 
	\leq 1 - \frac1{\sqrt{2\pi}} \paren{\frac{1}{\sqrt{\log (m) / 4}} - \frac{1}{(\log(m) / 4)^{3/2}}} \exp\paren{-\frac{\log m}{8}} \\
	&\leq 1 - \frac{1}{m^{1/6}}
	\leq 1 - \frac{1}{\sqrt{m}} -  \frac{\sqrt{2}c \rho}{n^{1/4}}.
\end{align*}

	Therefore, 
	\[\pr{\frac{1}{\norm[1]{\lambda}} y^\T \bX \bx_1' \leq t \mid \bxig} \leq 1 - \frac{1}{\sqrt{m}} \leq \exp\paren{-\frac{1}{\sqrt{m}}} \leq \paren{\frac{\delta}{3}}^{1/m},\]
	which above holds when $m \geq \parenl{\log \frac{3}{\delta}}^2$ and is ensured by a sufficiently large choice of $\cc$.
\end{proof}

The following well-known fact is the Mills ratio bound.
\begin{fact} \label{fact:mills}
Let $\Phi$ denote the standard Gaussian distribution function. Then for any $t \geq 0$,
\[ \left( \frac1t - \frac1{t^3} \right)  \cdot \frac1{\sqrt{2\pi}} e^{-t^2/2} \leq 1 - \Phi(t) \leq \frac1t \cdot \frac1{\sqrt{2\pi}} e^{-t^2/2} . \]
\end{fact}

\subsection{Modification of Theorem~\ref{thm:anisotropic} to support dependent labels}\label{assec:dependent-labels}

While an apparent weakness of Theorem~\ref{thm:anisotropic} is the fixed labels $y$, this can be surmounted.
Here, we outline how the proof can be easily modified to include labels $\by_i = \sign{v^\T \bx_i}$ for some unit vector $v$ for the isotropic Gaussian case; we believe this can be further generalized, but we present this version for the sake of simplicity.

\begin{theorem}[Lower-bound on SVP threshold with dependent labels]\label{thm:dependent-labels}
    Fix any $v \in \R^d$ with $\norml{v} = 1$, and consider an isotropic Gaussian sample $(\bX,\by)$ where each $\by_i = v^\T \bx_i$ and any $\delta \in (0, \half)$
	For absolute constants $\Ca, \Cb, \Cc, \Cd$, assume that $d$ and $n$ satisfy
\begin{equation}
	  n \geq \Ca \paren{\log \frac{1}{\delta}}^2, \quad d \leq \Cb n \log n, \quad \text{and} \quad
	  d \geq \Cc n \log\frac{1}{\delta}, 
	 \end{equation}
	Then\footnote{The fourth constraint is omitted because $d_2 = d_\infty = d$ in the isotropic case.}, SVP occurs for $\ell_2$-SVM with probability at most $\delta$.
\end{theorem}

The proof of the theorem relies on the same decomposition as Theorem~\ref{thm:anisotropic}.
The first step (Lemma~\ref{lemma:e1}) proceeds identically, because the proof of the lemma uses no other properties of the $\by$ besides the fact that it belongs to $\flip^n$.
The following inequality allows the remainder of the proof to proceed identically by fixing $y = \vec{1}$. 
For all $t \in \R$ and $i \in [n]$,
\[\pr{\by_i \by_{\setminus i}^\T \bX_{\setminus i} \bx_i \geq t} \geq \pr{\vec{1}^\T \bX_{\setminus i} \bx_i \geq t}.\]
We can prove this fact for the simple Gaussian setting by taking advantage of the fact that orthogonal components of a spherical Gaussian are independent.
For all $i$, we write $\bx_i = (v^\T \bx_i) v + \bx_i'$ where $v^\T \bx_i' = 0$.
Then
\begin{align*}
    \by_i \by_{\setminus i}^\T \bX_{\setminus i} \bx_i
    &= \sum_{j\neq i} (\by_j \bx_j)^\T(\by_i \bx_i)
    = \sum_{j\neq i} \sign{v^\T \bx_j v^\T \bx_i} [v^\T \bx_j v^\T \bx_i \norml[2]{v}^2 + \bx_j^{\prime \T} \bx_i'] \\
    &= \sum_{j\neq i}[\absl{v^\T \bx_j v^\T \bx_i} + \sign{v^\T \bx_j v^\T \bx_i}\bx_j^{\prime \T} \bx_i']
    \geq \sum_{j\neq i}[{v^\T \bx_j v^\T \bx_i} + \sign{v^\T \bx_j v^\T \bx_i}\bx_j^{\prime \T} \bx_i'].
\end{align*}

By independence and symmetry, each term in the last sum is distributed identically to ${v^\T \bx_j v^\T \bx_i} + \bx_j^{\prime \T} \bx_i' = \bx_j^\T \bx_i$. This gives the claim.

\section{Proofs for Section~\ref{sec:asymptotic}}\label{app:asymp-proofs}

In this section, we give the proof of Theorem~\ref{thm:asymptotic}.

\thmasymptotic*

We divide the proof into two cases, which we each prove in the two following subsections.

\subsection{Below the threshold}

We first consider the case where the dimension is below the threshold, specifically $d = (2-\eps_n) n \log n$.

Our proof follows the same strategy as that of Theorem~\ref{thm:anisotropic}.
Let $m :=  n / \log n$ and assume $n$ is sufficiently large so that $m \leq n/2$.
Using the equivalence from Proposition~\ref{prop:equiv}, it suffices to  show that the following event has probability tending to $1$:
\[
  \max_{i \in [m]}
  \bracket{
    y_i y_{\setminus i}^\T \paren{\bK_{\setminus i}^{-1}-\frac{1}{d} I_{n-1}} \bX_{\setminus i} \bx_i 
    + \frac{1}{d} y_i y_{[m] \setminus i}^\T \bX_{[m] \setminus i} \bx_i 
    + \frac{1}{d} y_i y_{\setminus [m]}^\T \bX_{\setminus [m]} \bx_i
  } \geq 1
  .
\]

\begin{lemma} \label{lem:asymptotic-lb1}
  For any $C_0>0$ and $c_1 \in (0,2)$, there exists $C_2>0$ and $n_0 > 0$ such that the following statements hold for all $n \geq n_0$, and all $\eps_n$ satisfying $2-c_1 \geq \eps_n \geq C_2/\sqrt{\log n}$ for all $n \geq n_0$, with $d = (2-\eps_n)n\log n$:
  \begin{enumerate}
    \item $\displaystyle \pr{\bK_{\setminus i} \text{ is invertible,} \ \max_{i \in [m]} \abs{y_i y_{\setminus i}^\T \paren{ \bK_{\setminus i}^{-1} - \frac{1}{d} I_{n-1} } \bX_{\setminus i} \bx_i } \le \frac{\eps_{n}}{2C_0} } \geq 1 - \frac1n$.

    \item $\displaystyle \pr{ \max_{i \in [m]} \abs{\frac{1}{d} y_i y_{[m] \setminus i}^\T \bX_{[m] \setminus i} \bx_i} \leq \frac{\eps_{n}}{2C_0} } \geq 1 - \frac1n$.

  \end{enumerate}
\end{lemma}
\begin{proof}
  We start with the first claim.
  By Lemma~\ref{lemma:e1} and a union bound, we have with probability at least $1 - 1/n$, $\bK_{\setminus i}$ is invertible and
  \begin{align*}
    \max_{i \in [m]}
    \abs{ y_i y_{\setminus i}^\T \paren{ \bK_{\setminus i}^{-1} - \frac{1}{d} I_{n-1} } \bX_{\setminus i} \bx_i }
    & \le C \sqrt{\frac{n \log(mn)}{d}}
    \paren{ \sqrt{\frac{n + \log(mn)}{d}} + \frac{n + \log(mn)}{d} } .
  \end{align*}
  For sufficiently large $n$, we have $d \geq c_1 n \log n$ and
  \begin{align*}
    C \sqrt{\frac{n \log(mn)}{d}}
    \paren{ \sqrt{\frac{n + \log(mn)}{d}} + \frac{n + \log(mn)}{d} }
    & \leq
    2C \sqrt{\frac{n \log(mn)}{d}} \sqrt{\frac{n + \log(mn)}{d}} \\
    & \leq \frac{3Cn\sqrt{\log n}}{(2-\eps_n)n\log n} + \frac{4C\sqrt{n}\log n}{c_1 n\log n}
  .
  \end{align*}
  The second term on the right-hand side is at most $C_2/(4C_0\sqrt{\log n})$ for sufficiently large $n$, and hence at most $\eps_n/(4C_0)$.
  The first term on the right-hand side is also at most $\eps_n / (4C_0)$ provided that
  \begin{equation*}
    (2-\eps_n) \eps_n \geq \frac{12C_0 C}{\sqrt{\log n}} ,
  \end{equation*}
  which is equivalent to
  \begin{equation*}
    \eps_n \geq 1 - \sqrt{1 - \frac{12C_0 C}{\sqrt{\log n}}}
  \end{equation*}
  (since we already assume $\epsilon_n \leq 2-c_1$ for sufficiently large $n$).
  This is satisfied provided that $C_2 \geq 12C_0 C$.
  This proves the first claim.

  For the second claim, we have by Lemma~\ref{lemma:e2} (with $n$ in the statement of Lemma~\ref{lemma:e2} set to $m-1$) and a union bound, we have with probability at least $1 - 1/n$:
  \begin{align*}
    \max_{i \in [m]}
    \abs{\frac{1}{d} y_i y_{[m] \setminus i}^\T \bX_{[m] \setminus i} \bx_i}
    & \leq 
    C \paren{\sqrt{\frac{m \log{m} + m\log(mn)}{d}} + \frac{\sqrt{m} \log(mn)}{d} } \\
    & \leq \sqrt{\frac{C'n}{d}} +\frac{C'n}{d\log n}
  \end{align*}
  where the second inequality uses $m = n / \log n \leq n/2$ and holds for sufficiently large $n$, with $C'>0$ an absolute constant.
  The second term on the right-hand side is at most $C_2/(4C_0\sqrt{\log n})$ for sufficiently large $n$, and hence also at most $\eps_n/(4C_0)$.
  The first term on the right-hand side is also at most $\eps_n/(4C_0)$ provided that
  \begin{equation*}
    \frac{16C_0C'}{\log n} \leq (2-\eps_n) \eps_n^2 .
  \end{equation*}
  Since $\eps_n \leq 2-c_1$, the above condition holds as long as $C_2 \geq 4\sqrt{C_0C'/c_1}$.
  This proves the second claim.
\end{proof}

\begin{lemma} \label{lem:asymptotic-lb2}
  For any $C_0>4$ and $c_1 \in (0,2)$, there exists $C_2>0$ such that the following holds for all sequences $(\eps_n)$ satisfying $2-c_1 \geq \eps_n \geq C_2/\sqrt{\log n}$ for all large enough $n$, with $d = (2-\eps_n)n\log n$:
  \begin{equation*}
    \lim_{n\to\infty}
    \pr{
      \max_{i \in [m]}
      \frac{1}{d} y_i y_{\setminus [m]}^\T \bX_{ \setminus [m]} \bx_i \geq 1 + \frac{\eps_{n}}{C_0}
    } = 1
    .
  \end{equation*}
\end{lemma}

\begin{proof}
  Observe that conditioned on $\bX_{\setminus [m]}$, the $m$ random variables
  \begin{equation*}
    \frac{1}{d} y_i y_{\setminus [m]}^\T \bX_{ \setminus [m]} \bx_i , \quad i = 1,\dotsc,m
  \end{equation*}
  are distributed as independent mean-zero Gaussian random variables with variance
  \begin{equation*}
    \sigma^2 := \frac{1}{d^2} \norml[2]{y_{\setminus [m]}^\T \bX_{ \setminus [m]}}^2 .
  \end{equation*}
  Therefore, the claim is equivalent to
  \begin{equation*}
    \lim_{n\to\infty} \pr{ \max_{i \in [m]} \sigma \bg_i \geq 1 + \frac{\eps_{n}}{C_0}} = 1 ,
  \end{equation*}
  where $\sigma^2$ is as defined above, and $\bg_1,\dotsc,\bg_m$ are i.i.d.~standard Gaussian random variables, independent of $\sigma^2$.

  Observe that
  \begin{equation*}
    \bG := \frac{1}{\sqrt{n-m}} \bX_{\setminus [m]}^\T y_{\setminus [m]}
  \end{equation*}
  is a standard Gaussian random vector in $\R^d$.
  By Gaussian concentration of Lipschitz functions~\cite[page 41, 2.35]{ledoux2001concentration}, the following holds with probability at least $1-1/n$:
  \begin{align}
    \sigma
    & \geq \frac{\sqrt{n-m}}{d}
    \paren{ \E{ \norm[2]{\bG} } - \sqrt{2\log n} }
    \nonumber \\
    & \geq \frac{\sqrt{n-m}}{d}
    \paren{ \sqrt{d} - \frac1{2\sqrt{d}} - \sqrt{2\log n} }
    \nonumber \\
    & = \sqrt{\frac{n-m}{d}}
    \paren{ 1 - \frac1{2d} - \sqrt{\frac{2\log n}{d}} }
    \nonumber \\
    & =
    \frac1{\sqrt{2\log n}}
    \frac1{\sqrt{1 - \frac{\eps_n}{2}}}
    \sqrt{1-\frac1{\log n}}
    \paren{ 1 - \frac1{2d} - \sqrt{\frac{2\log n}{d}} }
    \label{eq:good-variance-proxy}
  \end{align}
  where the second inequality follows from standard approximations of the Gamma function, and the final inequality holds assuming $C_2 \geq 1$.

  Let $\mathsf{E}$ be the event in which \eqref{eq:good-variance-proxy} holds.
  Then
  \begin{align*}
    \pr{ \max_{i \in [m]} \sigma \bg_i \geq 1 + \frac{\eps_n}{C_0} }
    & \geq
    \pr{ \max_{i \in [m]} \sigma \bg_i \geq 1 + \frac{\eps_n}{C_0} \mid \mathsf{E} } \paren{ 1 - \frac1n } \\
    & \geq
    \pr{ \max_{i \in [m]} \bg_i \geq \alpha_n \sqrt{2\log n} } \paren{ 1 - \frac1n }
  \end{align*}
  where
  \begin{equation*}
    \alpha_n := 
    \frac{\paren{ 1 + \frac{\eps_n}{C_0} }\sqrt{1 - \frac{\eps_n}{2}}}
    { \sqrt{1 - \frac{1}{\log n}} \paren{ 1 - \frac1{2d} - \sqrt{\frac{2\log n}{d}} } }
    .
  \end{equation*}
  We claim that the probability on the right-hand side tends to $1$ with $n\to\infty$ as well.

  The distribution of the random variable $\max_{i \in [m]} \bg_i$ obeys a limiting Gumbel distribution; specifically, for all $x>0$,
  \begin{equation*}
    \lim_{m\to\infty} \pr{
      \max_{i \in [m]} \bg_i \geq \sqrt{2\log m} - \frac{x + C \log \log m}{\sqrt{\log m}}
    }
    = 1 - e^{-e^x}
  \end{equation*}
  where $C>0$ is an absolute constant~\cite{fisher1928limiting}.
  Therefore, it suffices to show that for all $x>0$, we have
  \begin{equation*}
    \sqrt{2\log m} - \frac{x + C \log \log m}{\sqrt{\log m}}
    - \alpha_n \sqrt{2\log n} \geq 0
  \end{equation*}
  for all sufficiently large $n$.
  Dividing through by $\sqrt{2\log n}$ and using $m = n /\log n$, the above inequality is implied by
  \begin{equation}
    \sqrt{1 - \frac{\log\log n}{\log n}} - \frac{x + C \log \log (n/2)}{\log n} \frac1{\sqrt{1 - \frac{\log\log n}{\log n}}}
    - \alpha_n \geq 0 .
    \label{eq:gumbel-limit-ineq}
  \end{equation}
  Since $0 \leq \eps_n \leq 2-c_1$, we have
  \begin{equation*}
    \paren{ 1 + \frac{\eps_n}{C_0} }\sqrt{1 - \frac{\eps_n}{2}}
    \leq 1 - \frac{C_0-4}{2C_0\sqrt{2c_1}} \cdot \eps_n
  \end{equation*}
  by a Taylor series argument.
  So, \eqref{eq:gumbel-limit-ineq} is implied by
  \begin{equation*}
    \frac{C_0-4}{2C_0\sqrt{2c_1}} \cdot \eps_n - T(n)
    \geq 0
  \end{equation*}
  where
  \begin{multline*}
    T(n) =
    \paren{
      \frac{\log\log n}{\log n} - \frac{x + C\log\log(n/2)}{\log n} \paren{ 1 + \frac{\log\log n}{\log n} }
    }
    \\
    \cdot
    \paren{
      \sqrt{1 - \frac{1}{\log n}} \paren{ 1 - \frac1{2d} - \sqrt{\frac{2\log n}{d}} }
    }
    .
  \end{multline*}
  Since $\eps_n \geq C_2/\sqrt{\log n}$ and $T(n) = o(1/\sqrt{\log n})$, we can choose $C_2$ large enough so that \eqref{eq:gumbel-limit-ineq} holds for all sufficiently large $n$.
\end{proof}

We conclude as in the proof of Theorem~\ref{thm:anisotropic}.
The event in which all of the following hold has probability approaching $1$ as $n\to\infty$ by combining Lemma~\ref{lem:asymptotic-lb1} and Lemma~\ref{lem:asymptotic-lb2} and a union bound:
\begin{enumerate}
  \item $\max_{i \in [m]} \abs{y_i y_{\setminus i}^\T \paren{ \bK_{\setminus i}^{-1} - \frac{1}{d} I_{n-1} } \bX_{\setminus i} \bx_i } \le \frac{\eps_{n}}{2C_0}$;

  \item $\max_{i \in [m]} \abs{\frac{1}{d} y_i y_{[m] \setminus i}^\T \bX_{[m] \setminus i} \bx_i} \leq \frac{\eps_{n}}{2C_0}$;

  \item $\max_{i \in [m]} \frac{1}{d} y_i y_{\setminus [m]}^\T \bX_{ \setminus [m]} \bx_i \geq 1 + \frac{\eps_{n}}{C_0}$.

\end{enumerate}
In this event, there exists $i \in [m]$ such that
\begin{multline*}
  y_i y_{\setminus i}^\T \paren{\bK_{\setminus i}^{-1}-\frac{1}{d} I_{n-1}} \bX_{\setminus i} \bx_i 
  + \frac{1}{d} y_i y_{[m] \setminus i}^\T \bX_{[m] \setminus i} \bx_i 
  + \frac{1}{d} y_i y_{\setminus [m]}^\T \bX_{\setminus [m]} \bx_i
  \\
  \geq 
  -\frac{\eps_n}{2C_0} - \frac{\eps_n}{2C_0} + 1 + \frac{\eps_n}{C_0} = 1 .
\end{multline*}

\subsection{Above the threshold}

Now we consider the case where the dimension is above the threshold, specifically $d = (2+\eps_n) n \log n$.

By Proposition~\ref{prop:equiv}, it suffices to show that
\begin{equation*}
  \lim_{n\to\infty} \pr{
    \exists i \in [n] \ \ \text{such that} \ \
    y_i y_{\setminus i}^\T \bK_{\setminus i}^{-1}  \bX_{\setminus i} \bx_i \geq 1
  } = 0 .
\end{equation*}
This is implied by the following lemma combined with a union bound over all $i \in [n]$.

\begin{lemma}
  There exists $C_2>0$ and $n_0 > 0$ such that the following statement holds for all $n \geq n_0$, and all $\eps_n$ satisfying $\eps_n \geq C_2/\sqrt{\log n}$ for all $n \geq n_0$, with $d = (2+\eps_n)n\log n$:
  \begin{equation*}
    \text{for each $i \in [n]$:} \quad
    \pr{
      y_i y_{\setminus i}^\T \bK_{\setminus i}^{-1}  \bX_{\setminus i} \bx_i \geq 1
    }
    \leq
    \frac1{2n\sqrt{\pi\log n}} + \frac1{n^2} .
  \end{equation*}
\end{lemma}
\begin{proof}
  Conditional on $\bX_{\setminus i}$ (and the probability $1$ event that $\bX_{\setminus i}$ has rank $n-1$), the distribution of
  \begin{equation*}
    y_i y_{\setminus i}^\T \bK_{\setminus i}^{-1}  \bX_{\setminus i} \bx_i
  \end{equation*}
  is a mean-zero Gaussian with variance
  \begin{equation*}
    \sigma_i^2 := \norml[2]{y_{\setminus i}^\T \bK_{\setminus i}^{-1}  \bX_{\setminus i}}^2 = y_{\setminus i}^\T \bK_{\setminus i}^{-1}  y_{\setminus i} .
  \end{equation*}
  By Lemma~\ref{lemma:gram-evals}, we have for some absolute constant $C>0$ and sufficiently large $n$, with probability at least $1-1/n^2$,
  \begin{align*}
    \sigma_i^2
    & \leq n \cdot \mu_{\max}(\bK_{\setminus 1}^{-1})
    \\
    & \leq \frac{n}{d} \paren{ 1 + C \paren{ \sqrt{ \frac{n + 2 \log n}{d} } + \frac{n + 2 \log n}{d} } } \\
    & \leq \frac1{(2+\eps_n)\log n} \paren{ 1 + \frac{C'}{\sqrt{\log n}} } ,
  \end{align*}
  where $C'>0$ is a constant depending only on $C$.
  Let $\mathsf{E}$ be the aforementioned event.
  Then
  \begin{align}
    \pr{
      y_i y_{\setminus i}^\T \bK_{\setminus i}^{-1}  \bX_{\setminus i} \bx_i \geq 1
    }
    & \leq
    \pr{
      y_i y_{\setminus i}^\T \bK_{\setminus i}^{-1}  \bX_{\setminus i} \bx_i \geq 1 \mid \mathsf{E}
    } + \pr{ \neg \mathsf{E} }
    \nonumber \\
    & \leq
    1 - \Phi\paren{ \sqrt{\frac{(2+\eps_n)\log n}{1+C'/\sqrt{\log n}} } } + \frac1{n^2}
    \nonumber \\
    & \leq 
    \sqrt{\frac1{2\pi} \cdot \frac{1+C'/\sqrt{\log n}}{(2+\eps_n)\log n} }
    \exp\paren{ -\frac12 \cdot \frac{(2+\eps_n)\log n}{1+C'/\sqrt{\log n}} } + \frac1{n^2}
    \label{eq:asymp-ub}
  \end{align}
  where the final inequality follows by the Mills ratio bound (Fact~\ref{fact:mills}).
  By letting $C_2 \geq 2C'$, we have
  \begin{equation*}
    \frac{2+\eps_n}{1+C'/\sqrt{\log n}}
    \geq 2
  \end{equation*}
  (since we assume $\eps_n \geq C_2 / \sqrt{\log n}$), upon which bound in \eqref{eq:asymp-ub} is at most 
  \begin{equation*}
    \frac1{2n\sqrt{\pi\log n}} + \frac1{n^2}
  \end{equation*}
  as claimed.
\end{proof}


\section{Supplementary material for Section~\ref{sec:experiments}}
\label{app:supp-exp}

This appendix gives a refined statistical analysis of our hypothesis that SVP is universal for $\ell_{2}$-SVMs under the assumption that features are drawn identically and independently. 
We visually assert this universality with Figures~\ref{fig:heatmap-different-distributions} and \ref{fig:quantiles}, and formally test our hypothesis using a parametric statistical approach, borrowing several ideas from \citet{donoho2009observed}.

As described in Section~\ref{sec:experiments}, we show the significance of this universality by providing a parametric model that complies with the given universality hypothesis and fits well to the observed rates of SVP.
That is, this model permits slight difference for different sample distributions, as long as this difference decays to zero with $n$. 
Furthermore, we show that the model becomes statistically insignificant if we incorporate extra parameters to allow non-decaying dependence on sample distributions to conclude universality.

Alongside these universality results, we experimentally support the universality of the bounds on transition width, which are proved for the isotropic Gaussian sample case in Section~\ref{sec:asymptotic}. 

These analyses are implemented in Python and R. 
Our code-base can be found on Github at \url{https://github.com/scO0rpion/SVM-Proliferation-NIPS2021}.

\subsection{Experimental procedures} 
\label{app:exp-procedure}

We conduct a Monte Carlo simulation in order to validate our theoretical results and grasp their generality to distributions with different tail distributions. 
For the range $(n,d) \in \{40,42,\dots,100\} \times \{100,110,\dots,1000\}$ we study our problem in the following way:
\begin{itemize}
    \item We generate features $\bX \in \mathbb{R}^{n \times d}$ by drawing each $\bx_i$ independently from the suite of distributions shown in Table~\ref{tab:dist}.
    Subsequently, we generate a balanced set of labels $y \in \flip^{n}$ where the first $ \lfloor \frac{n}{2} \rfloor$ samples are assigned class $+1$ and the rest $-1$. 

    \item We used the quadratic program solver from CVXOPT~\cite{vandenberghe2010cvxopt}\footnote{CVXOPT is distributed at \url{https://cvxopt.org/} under a GPLv3 license.} to solve \eqref{eq:svm-restricted-dual-lp} with $p=q=2$ to tolerance level $10^{-7}$.
    
    \item We deem that SVP occurs for an instance of the simulation if the optimizer's output lies in the interior of $\mathbb{R}_{+}^{n}$.

    \item We report the fraction $\hat\bp$ of $M$ trials that exhibit SVP. Based on Figure~\ref{fig:Nsim-sensitivity}, we choose the simulation size value $M=400$ as an appropriate choice for having small enough variance for our range of $(n,d)$. 
    Throughout this section, we run $M = 400$ simulations unless stated otherwise. 
\end{itemize}

%
%

\begin{table}[t]
    \centering
    \def\arraystretch{1.2}
    \begin{tabular}{|c|c|c|c|c|}
        \hline
        \textbf{Name}
         & \textbf{Support}
         & \textbf{Mean}
         & \textbf{Variance}
         & \textbf{Subgaussian?} \\
         \hline
         Uniform
         & $[-1,1]$
         & $0$
         & $1/3$
         & Yes \\
         \hline
         Bernoulli
         & $\{0,1\}$
         & $1/2$
         & $1/4$
         & Yes \\
         \hline
         Rademacher
         & $\{-1,1\}$
         & $0$
         & $1$
         & Yes \\
         \hline
         Laplacian
         & $\R$
         & $0$
         & $2$
         & No \\
         \hline
         Gaussian
         & $\R$
         & $0$
         & $1$
         & Yes \\
         \hline
         Gaussian Biased
         & $\R$
         & $1$
         & $1$
         & Yes
         \\
         \hline
    \end{tabular}
    
    \smallskip

    \caption{\label{tab:dist}
    The suite of distributions used in experiments.
    Features $\bx_i \in \R^d$ are drawn from a product distribution $\distr^{\otimes d}$ of one of the distributions $\distr$ in this table.}
\end{table}

Our computing environment was a shared high-performance cluster, where a standard node has two Intel Xeon Gold 6226 2.9 GHz CPUs (each with 16 cores) and 192 GB memory.
It took roughly eight hours to run all SVP simulations for $\ell_2$-SVMs on a single node.
The simulations for $\ell_1$-SVMs for Figure~\ref{fig:l1-l2-compare} (Appendix~\ref{app:l1}) took two days to run on the cluster, again just on a single node.

\begin{figure}[t]
    \centering
    \includegraphics[width=\textwidth]{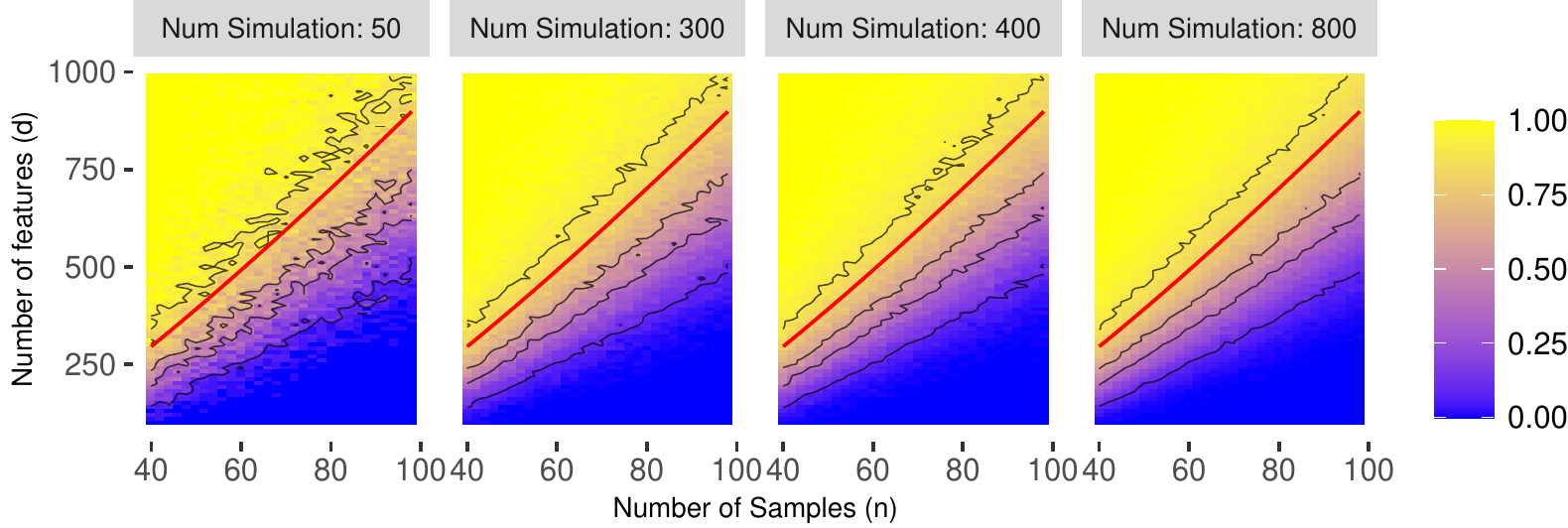}
    \caption{
    \label{fig:Nsim-sensitivity}
    The sensitivity of the experiments to simulation size $M$. 
    The blue curves are 0.1, 0.4, 0.6, and 0.9 quantile contours respectively for the observed rate of SVP for Gaussian samples.
    The red curve is limiting phase transition boundary, i.e., $n \mapsto 2 n \log(n)$}.
\end{figure}

\begin{figure}[t]
    \centering
    \includegraphics[scale=0.6]{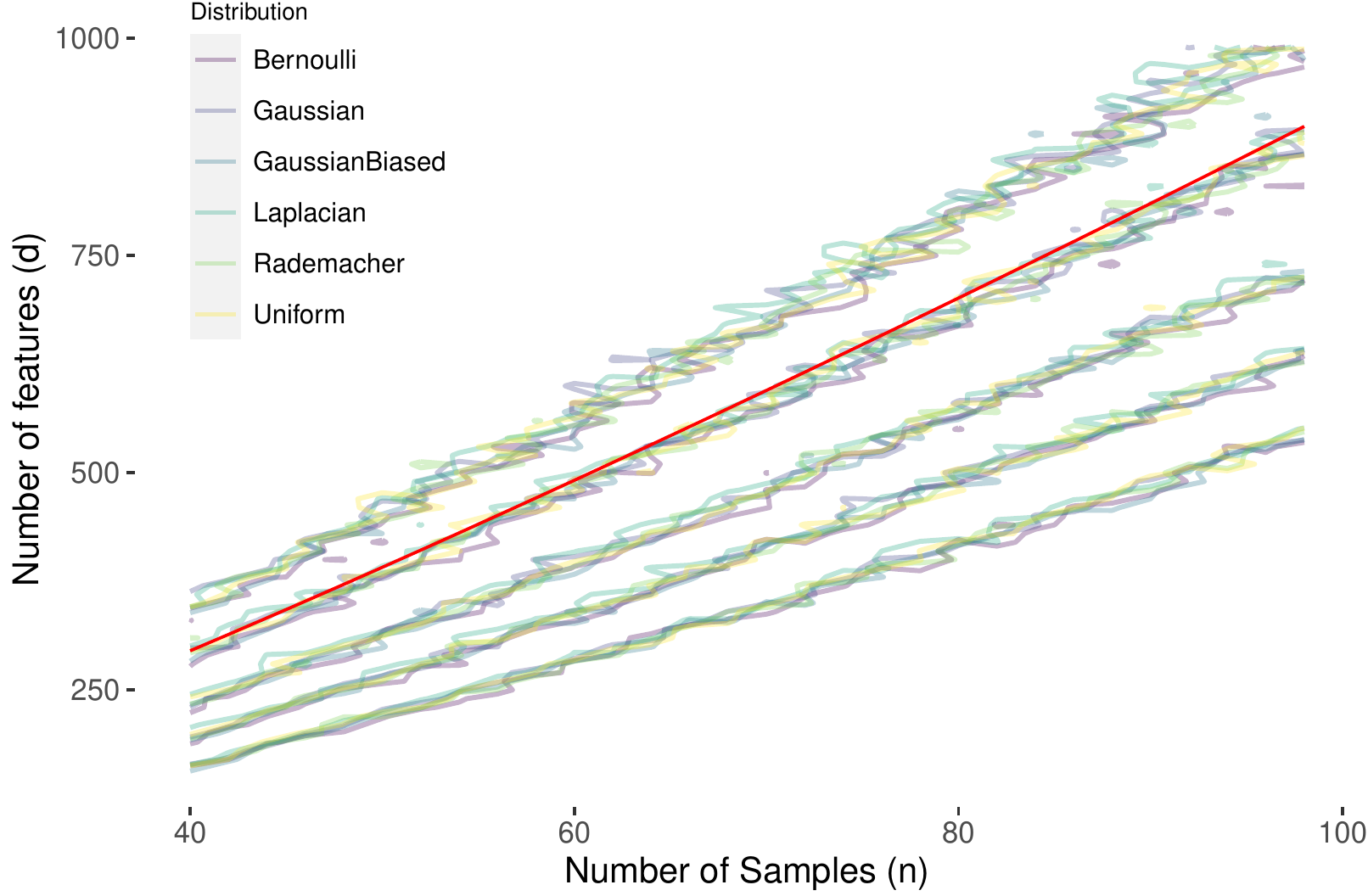}
    \caption{%
    \label{fig:quantiles}%
    Quantile plots. 
    The observed rates of SVP $\hat{\bp}$ are similar for all simulated distributions. 
    The quantile plots visualize the dimension $d$ needed for SVP to occur on a 0.2, 0.4, 0.6, 0.8, and 0.9 fraction of the trials for a fixed umber of samples $n$. 
    The red line corresponds to the asymptotic boundary $n \mapsto 2n\log n$, which closely aligns to the level curve corresponding to a 0.8 fraction of trials exhibiting SVP.
    }
\end{figure}

\subsection{Observed universality}\label{app-sub:universality}

Let $\hat{\mathbf{p}} \coloneqq \hat{\mathbf{p}}(n,d; \distr,M) \sim \text{Binom}(p(n,d; \distr), M)$ be the observed SVP rate corresponding to a sample distribution $\distr$ with independent components and with simulation size $M$. Due to the log-linear dependence of the dimension $d$ of SVP threshold occurs on $n$ from Theorem~\ref{thm:asymptotic}, we parameterize the probability of SVP as a function of $n$ and $\tau \coloneqq \fracl{d}{\paren{2n\log n}}$ instead. The objective here is to provide a reasonable parametric model for $\hat{\mathbf{p}}(n,d; \distr)$ as an inferential tool to test the universality hypothesis. To do so, we translate our universality hypothesis in the language of our parametric model and ensure that necessary statistical assumptions hold to make inferential claims.

\paragraph{Model:} 
We use Probit regression (a generalized linear model with Probit link function) to explain the transition behavior and allow the coefficients to depend explicitly on the distribution $\distr$ under which sample components are drawn from. 
We justify this specific choice of link function at the end of this subsection.
We propose the following parametric model; we motivate the terms in the model at the end of the subsection as well.
\begin{equation}
    \label{eq:model}
    p(n,d ; \distr) = \probit\paren{ \mu^{(0)}(n,\distr) + \mu^{(1)}(n, \distr) \tau + \mu^{(2)}(n,\distr) \log \tau  },
\end{equation}
where
\[
\mu^{(i)}(n, \distr) = \mu^{(i)}_{0}(\distr) + \frac{\mu^{(i)}_{1}(\distr)}{ \sqrt{n} },
\]
and
\[ \probit(t) = \int_{-\infty}^t \frac1{\sqrt{2\pi}} e^{-x^2/2} \operatorname{d}\! x \]
is the Probit link function (i.e., the standard normal distribution function).

The universality hypothesis can be translated to a testing framework under the model described \eqref{eq:model} by prohibiting $\mu^{(i)}_{0}(\distr)$ from depending on the underlying distribution $\distr$. 
\begin{itemize}
\item \textbf{Universality Hypothesis:} $\mu^{(1)}_0(\distr)$ are identical for each distribution $\distr$, and $\mu^{(1)}_0(\distr) = \mu^{(2)}_0(\distr) = 0$.
        
\item \textbf{Alternative Hypothesis:} $\mu^{(i)}_0(\distr)$ depends on the underlying distribution $\distr$ for some $i \in \{0,1,2\}$.
\end{itemize}

The Universality Hypothesis permits differences from the ground mean (which must decay to zero as $n$ grows large), but requires that other terms be identical.
The Alternative Hypothesis instead permits non-decaying difference among distributions. 
We show that our model \textbf{does not reject} the Universality Hypothesis. 

\paragraph{Inference:} We perform Probit regressions on observed SVP rates $\hat{\bp}$ sequentially on three different models, each of which is a sub-model of its successor. 
The second and third models correspond to the Universality Hypothesis and the Alternative Hypothesis respectively.
We compare their goodness-of-fit using analysis of deviance (ANOVA) to assess whether each subsequent model restriction meaningfully improves on its predecessor's ability to fit the data~\cite{hastie2017generalized}.

\begin{itemize}
    \item \textbf{Model~1:} Does not allow any deviations for different distributions; i.e., $\mu^{(i)}(n,\distr)$ does not depend on $\distr$ for $i \in \{0,1,2\}$. 
    \item \textbf{Model~2:} Allows deviations in bias that decay to zero; i.e., only $\mu^{(i)}_{1}(n,\distr)$ may vary with $\distr$.
    \item \textbf{Model~3:} Full model described in Equation~\eqref{eq:model}.
\end{itemize}

\begin{table}[t]
    \centering
    \def\arraystretch{1.2}
    \begin{tabular}{|c|c|c|c|}
    \hline
    \textbf{Models Compared}
    & \textbf{Degrees of Freedom}
    & \textbf{Deviance}
    & \textbf{P-Value} \\
    \hline
    1 vs.~2 & 5 & 1695.46 & $2\cdot10^{-16}$\\
    \hline
    2 vs.~3 & 25 & 29.25  & 0.253 \\
    \hline
    \end{tabular}

    \smallskip

    \caption{%
    \label{tab:anova}%
    Analysis of deviance for the sequence of three Probit models. P-Values are computed based on Chi-squared tails.}
\end{table}

Based on this sequential test given in Table~\ref{tab:anova}, we find that \textbf{Model~1} should be rejected, because \textbf{Model~2} substantially improves on it in a statistically significant manner. 
Furthermore, no statistical significance were detected for rejecting \textbf{Model~2}, and nearly all the excessive parameters (for \textbf{Model~3}) were statistically insignificant.
Therefore, we accept \textbf{Model~2}, which complies with the Universality Hypothesis.

Assuming that the Probit model is well-specified (e.g., the residuals satisfy the usual assumptions), we conclude that \textbf{Model 2} is significant.
The remainder of the section argues that the Probit model is the most appropriate for this setting. 
A full analysis of all three fitted models, including the fitted model parameters and p-values for these parameters, can be found in the code repository.



\paragraph{Motivating the model:} The proposed model \eqref{eq:model} is supported by a series of  empirical observations about the SVP rate $\hat{\bp}$:

\begin{itemize}
    \item $\hat{\bp}$ increases as $\tau$ increases and is mostly unaffected as $n$ changes (left panel (a) of Figure~\ref{fig:constant-slice}). 
    Therefore, the dependence on $n$ should be negligible.
    \item $\hat{\bp}$ is asymmetric around the theoretical boundary $\tau = 1$ (left panel (a) of Figure~\ref{fig:constant-slice}).
    This behavior motivates the non-linear term in our proposed model and likely originates from the asymmetry of the limiting Gumbel distribution in Section~\ref{sec:asymptotic}. 
    \item As $\tau$ varies, the slope of $n \mapsto p(n,\tau;\distr)$ changes, which motivates including terms that manage the interaction between $n$ and $\tau$ (right panel (b) in Figure \ref{fig:constant-slice}).
    \item As suggested by \citet{donoho2009observed}, including a dependence on $\fracl{1}{\sqrt{n}}$ is motivated by the theory of Edgeworth expansions, which states that the non-asymptotic behaviour of a random Gaussian problem decays to its asymptotic behavior by some power of the ``problem size.''
\end{itemize}

\paragraph{Model diagnostics:} 
While our proposed model yields formal evidence of universality, it remains to show that the model is correct, particularly because even wrong models are often statistically significant.
To avoid this pitfall, we validate the underlying assumptions required to make inferential claims in our logistic regression model. 
Figure~\ref{fig:diagnostics} demonstrates that \textbf{Model 2} accurately approximates the observed probabilities $\hat{\bp}$ when the probabilities are not too close to one or zero.\footnote{ 
These extreme cases are not concerning because we are primarily interested in values of $d$ and $n$ where the asymptotic probabilities are non-degenerate. 
The significance of \textbf{Model~2} continues to hold, even if we restricted attention to a smaller region with non-degenerate SVP rates $\hat{\bp}$, such as $\tau \in [0.4,1.6]$.} 
The figure further shows that the residuals approximately follow a standard normal distribution and do not seem to correlate with the fitted values. 
Therefore, the residuals corresponding to this model satisfy the usual assumptions necessary for regression.

\begin{figure}[t]
    \centering
    \begin{tabular}{cc}
        \includegraphics[width=0.48\textwidth]{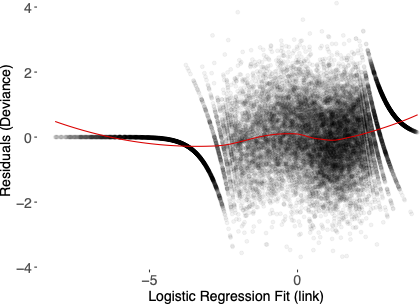} &
        \includegraphics[width=0.48\textwidth]{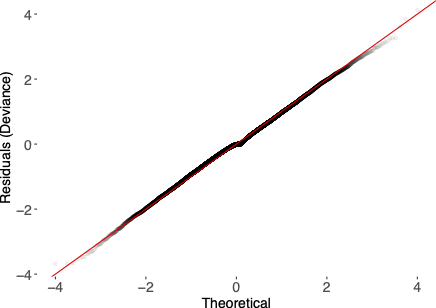} \\
        (a) & (b) \\
    \end{tabular}
    \caption{%
    Diagnostic plots.
    \emph{Left panel (a):} The relationship between logistic regression residuals and fitted probabilities for \textbf{Model 2} applied to the inverse link function. 
    The red curve is LOESS smoothing to illustrate the trend. 
    The observed probabilities drop to zero faster than tails of Probit function.
    \emph{Right panel (b):} Quantile plot of the residuals.
    The red line corresponds to standard normal quantiles.%
    \label{fig:diagnostics}
    }
\end{figure}

\paragraph{Justification for the link function and log:} We test three link functions for \textbf{Model 2} to choose an appropriate one for our parametric model~\eqref{eq:model}:
\begin{align*}
    \operatorname{Cauchit}(t) & = \frac1\pi \left[ \tan^{-1}(t) + \frac\pi2 \right] ,
    & \operatorname{Logit}(t) & = \frac1{1+e^{-t}} ,
    & \operatorname{Probit}(t) & = \Phi(t) .
\end{align*}
We justify the use of $\log{\tau}$ as a non-linear term in \eqref{eq:model} by substituting the logarithmic term with the \emph{Cox-Box transform} of $\tau$, i.e., $\fracl{\paren{\tau^{\gamma} - 1}}{\gamma}$ and cross-fitting various link functions with different exponents $\gamma = 0.2, 0.4, 0.5, 1$. 
The diagnostic plots in Figure~\ref{fig:appropriate-link-diagnostic-res-vs-fitted} indicate that the Probit link function fits best. 
Moreover, smaller values of $\gamma$ fit observed probabilities better by comparing the deviance r-squares in Table~\ref{table:Rsquared-link} as a measure of goodness of fit. 
Hence, we use the $\log \tau$, which is the limit of the Cox-Box transform when $\gamma \to 0$.

\begin{figure}[t]
    \centering
    \begin{tabular}{cc}
    \includegraphics[width=0.48\textwidth]{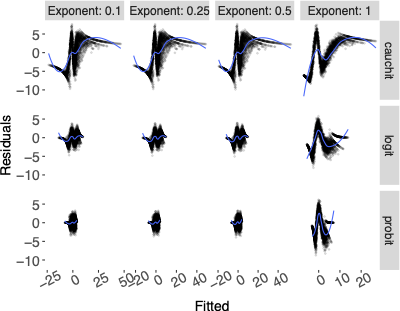}
    &
    \includegraphics[width=0.48\textwidth]{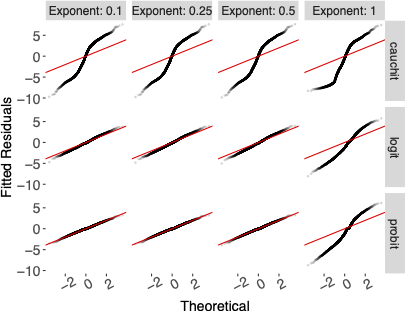}
    \\
    (a) & (b)
    \end{tabular}
    \caption{
    \label{fig:appropriate-link-diagnostic-res-vs-fitted}
    Diagnostic plots for Gaussian distribution.
    \emph{Left panel (a):} The relationships between the fitted values from logistic regression and the residuals for various link functions and exponents. We use LOESS smoothing to visualize the trend and plot it in blue. Asymptotically, the residuals should appear independent of the fitted values; hence, bumpy blue curves indicate non-compliant behaviour with model assumptions.
    \emph{Right panel (b):} Residual QQ-plots. One expects residuals to asymptotically follow a normal distribution (with a perfect fit corresponding to the overlaid red line), so smaller deviations from a normal distribution is desired.}
\end{figure}

\begin{table}[t]
    \centering
    \def\arraystretch{1.2}
\begin{tabular}{cc|c|c|c|}
\cline{3-5}
\multicolumn{1}{l}{}                                         & \multicolumn{1}{l|}{} & \multicolumn{3}{c|}{\textbf{Link}}                  \\ \cline{3-5} 
\multicolumn{1}{l}{}                                         & \multicolumn{1}{l|}{} & \textbf{Cauchit} & \textbf{Logit} & \textbf{Probit} \\ \hline
\multicolumn{1}{|c|}{\multirow{4}{*}{$\boldsymbol{\gamma}$}} & \textbf{0.10}         & 0.9606           & 0.9953         & 0.9971          \\ \cline{2-5} 
\multicolumn{1}{|c|}{} & \textbf{0.25} & 0.9602 & 0.9950 & 0.9969 \\ \cline{2-5} 
\multicolumn{1}{|c|}{} & \textbf{0.50} & 0.9595 & 0.9946 & 0.9967 \\ \cline{2-5} 
\multicolumn{1}{|c|}{} & \textbf{1.00} & 0.9501 & 0.9771 & 0.9722 \\ \hline
\end{tabular}

    
    \smallskip
    
    \caption{%
    \label{table:Rsquared-link}
    $R^{2}$ values computed from on the fraction between the null deviance and the fitted deviance for different link functions and exponents. The Probit link function with  $\gamma=0.1$ has the best fit.}
\end{table}

\subsection{Width of transition}\label{app-sub:width}

\begin{figure}[t]
    \centering
    \includegraphics[scale = 0.7]{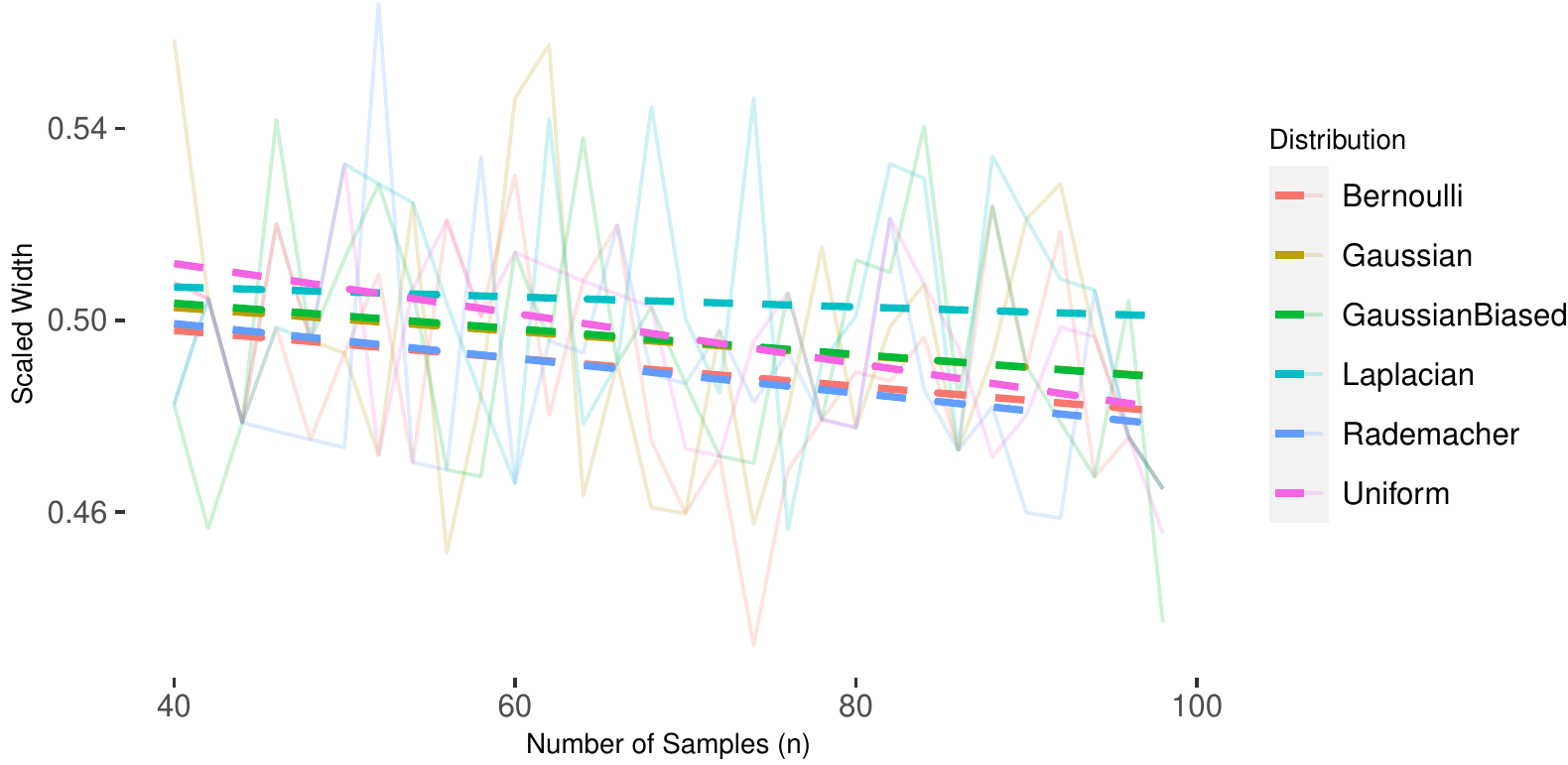}
    \caption{
    \label{fig:transition-width}
    The estimated scaled transition widths for variety of distributions. 
    Note that $\hat{\tau}_{q}$ and $\hat{\tau}_{1-q}$ are computed with both observed probabilities and non-parametric smoothing spline versions of probabilities which are shown in left and right panels respectively.
    The overall trend is overlaid using a linear regression applied to the scale width.}
\end{figure}

To estimate the width of the phase transaction, we adopt a non-parametric approach. 
For $q < \fracl{1}{2}$ and fixed $n$ we define the \emph{$q$-transition zone} to be the range $[\tau_{q},\tau_{1-q}]$, where $\tau_q$ and $\tau_{1-q}$ correspond to the ratio $\fracl{d}{\paren{2n \ln n}}$ for smallest and largest value of $d$ where support vector proliferation occurs with probability $q$ and $1-q$ respectively. In other words, within this range, the corresponding probabilities are inside $[q, 1-q]$ interval. 

Formally, we define \emph{scaled $q$-transition width estimate} as,
\[
\hat{w}_{q}(n,d) = \frac{ \hat{\tau}_{1-q} - \hat{\tau}_{q} }{ \probit^{-1}(1-q) - \probit^{-1}(q) } \sqrt{\log n},    
\]
where $\probit$ is the Probit link function introduced in \eqref{eq:model} and $\hat{\tau}_{q}$ and $\hat{\tau}_{1-q}$ are plug-in estimates of $\tau_{q}$ and $\tau_{1-q}$.  
We plot $\hat{w}_{q}(n,d)$ with $n$ in Figure~\ref{fig:transition-width}.



\section{Supplementary material for Section~\ref{sec:l1}}\label{app:l1}

\subsection{$\ell_1$ SVM experimental design}\label{app-sub:l1-exp}
To generate Figure~\ref{fig:l1-l2-compare}, we use $M=400$ Monte Carlo simulations for each pair $(n,d)$ for both $\ell_{1}$ and $\ell_{2}$ SVMs with Gaussian ensemble features $\bX$ and labels $y$ generated identically to what discussed in Appendix~\ref{app:supp-exp}. The range of values of $(n,d)$ for $\ell_{1}$-SVM is within $\{10,11,\dots,70\}\times \{80,85,\dots,1000\}$.
The remainder of the $\ell_1$ experiment is identical to the aforementioned $\ell_2$ experiments.
We determine whether SVP occurs for the $\ell_1$ case by solving the \eqref{eq:svm-dual-l1} linear program for a $\bX$ and $y$ and identifying whether the resulting $\alpha^{*}$ is in $\R_{+}^{n}$. 
We use the linear program solver in CVXOPT~\cite{vandenberghe2010cvxopt} with the default configuration (\texttt{absolute tolerance} $= 10^{-7}$, \texttt{relative tolerance} $= 10^{-6}$, \texttt{feasibility tolerance} $= 10^{-7}$).

\subsection{A geometric interpretation of $\ell_1$ SVM proliferation}\label{app-sub:l1-geometric}


A persistent line of work in the statistics and machine learning literature studies phenomena with $\ell_1$ constraints by geometrically analyzing the random mathematical programs.
One such application is \emph{exact sparse recovery}, which assumes the existence of a sparse ``ground truth'' that one aims to recover using linear observations by solving a convex optimization problem.
\citet{donoho2009counting} and \citet{almt14} show the existence of a phase transition in this and similar problems by translating the optimality conditions into a geometrical question as to whether two cones share a ray. 
Motivated by this line of work, we believe that Conjecture \ref{main-conjecture} at its core is a geometric phenomenon.
Along those lines, we translate the problem of determining whether $\alpha \in \R_+^n$ into a geometric problem that aims to characterize the faces of a random polytope.

By the Fundamental Theorem of Linear Programming, the optimal solution(s) to \eqref{eq:svm-dual-l1} lie on corner points of the polytope \[ \mathcal{C}^{*} = \{ \alpha \in \mathbb{R}^{n} : \|\bA^\T \alpha \|_{\infty} \le 1\} \]
in the dual space. 
The following lemma provides an alternate characterization the optimal solution $\alpha^*$ to the linear program by relating the corner points of $\mathcal{C}^*$ to the faces of its dual polytope, $\mathcal{C}$.

\begin{lemma}
\label{lem:l1-optimality}
 Suppose $\bA \in \mathbb{R}^{n \times d}$ is a full rank matrix for $d > n$. 
 Then $\alpha^{*}$ is an optimal solution to \eqref{eq:svm-dual-l1} iff it is perpendicular to a facet of $\mathcal{C}=\operatorname{Conv}\{\pm \bA_{.,i} : i \in [d]  \}$ which intersects with the ray passing through origin and $\mathbf{1} \in \mathbb{R}^{n}$. 
\end{lemma}

As a result of Lemma~\ref{lem:l1-optimality}, we can determine whether $\ell_1$ support vector proliferation occurs by instead determining whether the ray in the direction of $\one$ passes through a facet of $\mathcal{C}$ whose closest point to the origin lies in $\R_+^n$. 
This reduces the problem to understanding the geometry of the facets of random polytope $\mathcal{C}$. If there are sufficiently many facets relative to $2^n$ and most facets cover a very small number of orthants, then it becomes increasingly likely that the facet intersected by the $\one$ ray will also have its projection from $\zero$ lie in the positive orthant and support vector proliferation is likely to occur. This intuition can be supported from Dvortzky-Milman Theorem \cite{dvoredsky1961some} since for a Gaussian ensemble $\bA$ one expects the convex hull of its columns to be isomorphic to a sphere when $d = \exp\paren{\Omega(n)}$. This geometric approach informs our conjecture that the phase transition is likely to occur at the rate $f(n) = \exp\paren{\Theta(n)}$.

Figure~\ref{fig:l1-example} for a Gaussian Ensemble with $n=2$ and $d=1500$ and illustrates the relationships between $\mathcal{C}$, $\mathcal{C}^*$, $\alpha^*$, and $\bA_{.,i}$. 
The geometric intuition conveyed in the figure is limited by its low value of $n$; we expect qualitatively different behavior to arise when $n$ is large, particularly when considering the geometry of the facets of $\mathcal{C}$.
For instance, when $n = 2$, very few samples correspond to corners of the convex hull, and the vast majority lie in the interior.
When $n$ is larger, almost all of the samples will be corners of the convex hull, unless $d$ grows exponentially with $n$.

\begin{figure}[t]
    \centering
    \includegraphics[width=0.6\textwidth]{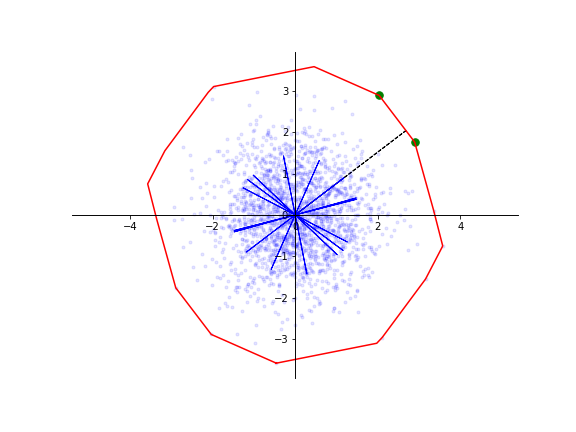}
    \caption{%
    \label{fig:l1-example}%
    Depiction of the dual space. 
        Blue points correspond to rows of a Gaussian ensemble and their reflections with respect to origin, i.e. $\pm \bA_{i,:}$. The outer red convex hull represents $\mathcal{C}$.
        Blue lines are corner points of $\mathcal{C}^{*}$, which are perpendicular to the faces of $\mathcal{C}$, and the black line is the direction of the optimal solution $\alpha^{*}$. 
        Green points indicate the active constraints where the inner product with $\alpha^{*}$ is equal to one.
        }
\end{figure}

To prove Lemma~\ref{lem:l1-optimality}, we make use of several useful facts about the geometric properties of dual polytopes, which are immediate consequences of Farkas' Lemma~\cite[see][]{matousek2013lectures}.

\begin{fact}
    \label{fact:polytope-duality}
    Let $\mathcal{C} = \operatorname{Conv}\{\pm A_{.,i} : i \in [d]\}$ and $\mathcal{C}^{*}$ defined as above. 
    Then for a full rank matrix $\bA \in \mathbb{R}^{n \times d}$ the following holds:
    \begin{itemize}
        \item The origin $\zero$ is in the interior of $\mathcal{C}$ and $\mathcal{C}^{*}$.
        \item $\mathcal{C} = (\mathcal{C}^{*})^{*} \coloneqq \{\alpha \in \mathbb{R}^{n} : \langle \alpha^{\prime} , \alpha \rangle \leq 1  , \ \forall \alpha^{\prime}  \in \mathcal{C}^{*} \}$, in words $\mathcal{C}$ and $\mathcal{C}^{*}$ are polar.
        \item Corner points of $\mathcal{C}^{*}$ are perpendicular to facets (boundary hyperplanes) of $\mathcal{C}$.
    \end{itemize}
\end{fact}

Now, we are ready to prove Lemma~\ref{lem:l1-optimality}.

\begin{proof}[Proof of Lemma~\ref{lem:l1-optimality}]
We consider the \eqref{eq:svm-dual-l1} optimization problem and move constraints into the objective for some Lagrange multiplier $\beta \in \mathbb{R}_{+}$:
\[
\max_{\alpha \in \mathbb{R}^{n}}
\mathbf{1}^\T \alpha + \beta\left( 1 - \max_{i \le 2d}{ \bB_{i}^\T \alpha } \right)
\]
where $\bB_{i} = -\bB_{i+d} = \bA_{\cdot,i} \in \R^n$ is the $i$th column of $\bA$.

If $\alpha^*$ is an optimal solution, then zero must be a subgradient of the objective at $\alpha^*$.
Since the set of subgradients of $\alpha \mapsto \max_{i \in [2d]} \bB_i^\T \alpha$ at $\alpha$ is $\operatorname{Conv}\{ \bB_i : i \in I_{\alpha} \}$, where $I_{\alpha} = \{i \in [2d] : \langle \bB_{i}, \alpha \rangle = \max_{i' \in [2d]}{ \langle \bB_{i'}, \alpha \rangle }) \}$ denotes the set of active constraints at $\alpha$, an optimal solution $\alpha^*$ must satisfy
\[
\frac{1}{\beta}\mathbf{1} \in \operatorname{Conv}\{ \bB_{i} : i \in I_{\alpha^{*}}\}.
\] 
This convex hull is a face of $\mathcal{C}$ since $\langle \bB_{j}, \alpha^{*} \rangle < \langle \bB_{i}, \alpha^{*} \rangle = 1$ for any $j \notin I_{\alpha^{*}}$ and $i \in I_{\alpha^{*}}$. 
Moreover, since $\alpha^{*}$ is also a corner point of $\mathcal{C}^{*}$ we have $|I_{\alpha^{*}}| = n$. 
Combining with Fact~\ref{fact:polytope-duality}, we conclude that $\alpha^{*}$ must be perpendicular to the facet of $\mathcal{C}$ that intersects with the ray passing through $\mathbf{1}$. Existence of such a facet is ensured by the fact that origin is in the interior of $\mathcal{C}$. 
\end{proof}

\subsection{Lower-bound on dimension needed for $\ell_1$ SVP}\label{app-sub:l1-loose}

Based on the previous relationship between support vector proliferation and the faces of a random convex polytope, we can prove a very loose bound on the minimum dimension $d$ needed for support vector proliferation to occur by bounding the number of faces of the polytope.
When $d = O\paren{n}$, the polytope will have far fewer than $2^{n}$ facets, which makes it impossible for most orthants to to be covered by projections from $\zero$ onto facets. This (along with rotational invariance of the standard Gaussian distribution) allows us to show that support vector proliferation occurs with a negligibly small probability in this regime.

\begin{theorem}\label{thm:weak-ub}
    Let $\bA \in \mathbb{R}^{n \times d}$ be a matrix of i.i.d.~$\mathcal{N}(0,1)$ random variables with $d \geq n$.
    Let $\mathbf{S}$ denote the set of maximizers of \eqref{eq:svm-dual-l1}.
    If $d < C n$ for some universal constant $C>1$, then
    \[
    \limsup_{n \to \infty}{ \pr{\mathbf{S} \cap \operatorname{int}(\mathbb{R}_{+}^{n}) \neq \emptyset }} = 0 
    .
    \]
\end{theorem}


A key geometric insight for the proof is supplied by \cite{donoho2009counting}.
Let $k^*$ be the maximum integer $k$ such that every $k$ points from among $\{ \pm \bA_{\cdot,i} : i \in [d] \}$ spans a $k$-dimensional face of $\mathcal{C}$ (i.e., their convex hull is a $k$-dimensional face of $\mathcal{C}$).
Then $k^{*}= \Omega_{\mathbf{P}}(\frac{n}{\log(d/n)})$ for large enough values of $n$ and $d$.
Hence, it is plausible to expect that every selection of $n$ points spans a facet of $\mathcal{C}$.


\begin{proof}
Let $F$ denote the facets of $\mathcal{C}$, and to each facet $f \in F$, we associate a corner point $\alpha^{(f)}$ of $\mathcal{C}^*$.
Also let $\bU \in \R^{n \times n}$ denote a uniformly random rotation matrix, independent of $\bA$.
Note that the collection $(\alpha^{(f)})_{f \in F}$ will also have a rotational invariant distribution.
Since $\bA$ has rank $n$ almost surely, Lemma~\ref{lem:l1-optimality} implies that an optimal solution $\alpha^*$ to \eqref{eq:svm-dual-l1} must be one of these corner points $\alpha^{(f)}$.
Thus, we can upper bound the event that any given optimal solution $\alpha^*$ to \eqref{eq:svm-dual-l1} lies in the positive orthant as follows: 
\begin{align*}
    \pr{\alpha^{*} \in \operatorname{int}(\mathbb{R}_{+}^{n})}
    & \le \pr{
        \bigcup_{f \in F}
        \{ \alpha^{(f)} \in \operatorname{int}(\mathbb{R}_{+}^{n}) \}
    }
    \\
    & =
    \E{
        \pr{
            \bigcup_{f \in F}
            \{ \bU\alpha^{(f)} \subset \mathbb{R}_{+}^{n} \}
            \,\Big\vert\, (\alpha^{(f)})_{f \in F}
        }
    } \\
    & \leq
    \E{
        \sum_{f \in F}
        \pr{
            \{ \bU\alpha^{(f)} \subset \mathbb{R}_{+}^{n} \}
            \,\Big\vert\, (\alpha^{(f)})_{f \in F}
        }
    } \\
    & \le \frac{\E{|F|}}{2^{n}}
    .
\end{align*}
The second inequality is a union bound, and the final inequality uses the fact that $\mathbf{P}[\bU\alpha \in \mathbb{R}_{+}^{n}] = 1/2^{n}$ for any fixed $\alpha$.
A crude upper bound of $\binom{d}{n}$ on the number of facets gives
\[
\mathbf{P}[\alpha^{*} \in \mathbb{R}_{+}^{n}] \le \frac{{d \choose n}}{2^{n}} .
\]
The right-hand side converges to zero as $n \to \infty$ provided that $d < Cn$ for some absolute constant $C \approx 1.29$.
\end{proof}

Needless to say, the gap between Theorem~\ref{thm:weak-ub} and the $\ell_1$ support vector proliferation threshold exhibited in Figure~\ref{fig:l1-l2-compare} is substantial.
Indeed, if Theorem~\ref{thm:weak-ub} were tight, it would imply that support vector proliferation would occur for \emph{smaller} values of $d$ for $\ell_1$ than $\ell_2$, which contradicts our experimental results and geometric intuition.
We believe that union bound corresponding to the first inequality in our proof accounts for that looseness.
That inequality would be tight only if the existence of some $\alpha^{(f)} \in \operatorname{int}(\R_+^n)$ implies that $\alpha^* \in \operatorname{int}(\R_+^n)$; however, this ignores the possibility that facets span many different orthants and that the facet intersecting the ray through $\one$ may be ``close to the origin'' in many orthants simultaneously.
We believe that a more precise understanding of the geometry of the faces of random high-dimensional polytopes could tighten this bound and hence elucidate the $\ell_1$ support vector proliferation phenomenon.

\end{document}